%% file: main_arxiv.tex
\documentclass[11pt]{article}

\usepackage{amsfonts}
\usepackage{xcolor}
\usepackage{amsthm}
\usepackage[normalem]{ulem}

\usepackage{amsmath}
\usepackage{amssymb}
\usepackage[numbers,sort&compress]{natbib}

\usepackage{multirow}
\usepackage{mathtools}
\usepackage{graphicx}
\usepackage{url}
\usepackage[ruled]{algorithm2e}
\PassOptionsToPackage{algo2e,ruled}{algorithm2e}
\usepackage{amsthm}
\usepackage[margin=1in]{geometry}
\usepackage{hyperref}
\usepackage[capitalize]{cleveref}
\hypersetup{colorlinks=true,linkcolor=blue,citecolor=blue}

\definecolor{darkgreen}{RGB}{0,100,0}
\definecolor{forestgreen}{RGB}{40,180,40}
\usepackage[mathscr]{eucal}
\usepackage{thm-restate}

\newtheorem{theorem}{Theorem}
\newtheorem{lemma}[theorem]{Lemma}
  
    \newtheorem{definition}{Definition}
  \newtheorem{proposition}[theorem]{Proposition}

\usepackage{pgfplots}
\pgfplotsset{compat=1.18}
\usetikzlibrary{patterns}

\title{Characterizing Online and Private Learnability under Distributional Constraints via Generalized Smoothness}

\author{
   Mo\"ise Blanchard\footnote{equal contribution}\\
  GeorgiaTech\\
  \small{\texttt{mblanchard41@gatech.edu}}
  \and
  Abhishek Shetty\footnotemark[1]\\
  MIT\\
  \small{\texttt{shetty@mit.edu}}
  \and
  Alexander Rakhlin\\
  MIT\\
  \small{\texttt{rakhlin@mit.edu}
  }
}

\date{}

\usepackage{thmtools}
\usepackage{thm-restate}
\usepackage[mathscr]{eucal}
\usepackage{bbm}

\newcommand{\acks}[1]{\section*{Acknowledgments}#1}

\newcommand{\nonl}
{\renewcommand{\nl}{\let\nl\oldnl}}

\usepackage{latexsym,tikz,graphicx}
\usetikzlibrary{calc}
\usetikzlibrary{decorations.pathreplacing}
\usetikzlibrary{patterns}

\renewenvironment{proof}[1][]{\par\noindent{\bf Proof #1\ }}{\hfill$\blacksquare$\\[2mm]}

\begin{document}

\input{shortcuts}

\maketitle

\begin{abstract}%

Understanding minimal assumptions that enable learning and generalization is perhaps the central question of learning theory. 
Several celebrated results in statistical learning theory, such as the VC theorem and Littlestone's characterization of online learnability, establish conditions on the hypothesis class that allow for learning under independent data and adversarial data, respectively. 
Building upon recent work bridging these extremes, we study sequential decision making under \emph{distributional adversaries} that can adaptively choose data-generating distributions from a fixed family $\Ucal$ and ask when such problems are learnable with sample complexity that behaves like the favorable independent case. 
We provide a near complete characterization of families $\Ucal$ that admit learnability in terms of a notion known as \emph{generalized smoothness} i.e.  
a distribution family admits VC-dimension-dependent regret bounds for every finite-VC hypothesis class if and only if it is generalized smooth. 
Further, we give universal algorithms that achieve low regret under any generalized smooth adversary without explicit knowledge of $\Ucal$. 
Finally, when $\Ucal$ is known, we provide refined bounds in terms of a combinatorial parameter, the fragmentation number, that captures how many disjoint regions can carry nontrivial mass under $\Ucal$.
These results provide a nearly complete understanding of learnability under distributional adversaries. 
In addition, building upon the surprising connection between online learning and differential privacy, we show that the generalized smoothness also characterizes private learnability under distributional constraints.  
\end{abstract}


\section{Introduction}

The grand challenge of statistics and learning theory is to understand when data can be used to make accurate predictions.
Various formalizations of this question have been studied over the decades, leading to a rich theory of learnability under different data-generating assumptions on the covariates and the labels, performance metrics, feedback models, and computational constraints.

One prominent framework is \emph{PAC learning} \citep{valiant1984theory,vapnik1974theory} which supposes that data are drawn independently and identically from a distribution. 
In this setting, the celebrated VC theorem \citep{vapnik1971uniform,vapnik1974theory} shows that a hypothesis class $\Fcal$ is learnable with sample complexity depending on its VC dimension \citep{vapnik1971uniform}.
On the other extreme of the data-generating spectrum lies \emph{online learning} \citep{littlestone1988learning}, which models data as being completely arbitrarily chosen (modeled as being chosen by an adversary).
Here, Littlestone's seminal work \citep{littlestone1988learning} established that a hypothesis class $\Fcal$ is learnable with regret depending on its Littlestone dimension \citep{littlestone1988learning}.  
Though these two frameworks provide a solid understanding of learnability under independent data (which we view as ``benign'' or well-behaved in our present context but consider a strong and perhaps unreasonable modelling assumption) and adversarial data (which is a minimal assumption that treats the data as malicious and ill-behaved), real-world applications typically involve data-generating processes that lie between these two extremes.
Further, and perhaps unsurprisingly, the characterization of learnability under the fully adversarial model of Littlestone precludes learnability even for simple hypothesis classes such as linear thresholds \citep{littlestone1988learning}. 

This motivates the study of structured data-generating processes that interpolate between the two extremes of i.i.d.\ and fully adversarial data.  
A natural formulation is to consider a process where the choice of data at each round is constrained to lie within a certain family of distributions. 
This framework is referred to as \emph{distributionally constrained online learning}. 
Perhaps the most well-studied special case is that of \emph{smoothed adversaries} \citep{rakhlin2011online,haghtalab2024smoothed,shetty2024learning,block2022smoothed}, where the adversary is constrained to choose distributions that have bounded density ratio with respect to a fixed reference measure. 
This line of work, which can be also be seen as a model for distribution shift (See \citep{hanneke2025adaptivesampleaggregationtransfer} and references therein),  has established that, under smoothness, sequential decision-making has the same qualitative learnability guarantees as in the i.i.d.\ setting \citep{haghtalab2024smoothed}. 
Though this gives a satisfactory answer under smoothness, perhaps the condition is too strong, leading to the following fundamental question:
\begin{quote}
\emph{What are minimal conditions on a distribution family $\Ucal$ that ensure any form of learnability for low-complexity function classes (e.g., finite VC dimension), against adaptive $\Ucal$-distributionally constrained adversaries?}
\end{quote}

Surprisingly, this question is intimately connected to a seemingly unrelated problem of learning under \emph{privacy} constraints. 
Differential privacy \citep{dwork2006calibrating} is a rigorous framework for ensuring that learning algorithms do not leak sensitive information about individual data points. 
Unfortunately, differential privacy, being a stringent condition often imposes significant limitations on what can be learned, perhaps most notably exemplified by the fact that for binary classification, only classes with finite Littlestone dimension are privately learnable \citep{alon2019private}.  
As with online learning, this negative result has motivated the study of structured data-generating processes that allow for private learnability of richer classes \citep{haghtalab2020smoothed,alon2019limits,bassily2020private,beimel2014bounds}. 
We ask whether there is a characterization of distribution families that admit private learnability with finite sample complexity for VC hypothesis classes.

\subsection{Our Contribution.}

We address this question by providing a complete characterization of distribution families $\Ucal$ that admit learning for VC classes. 
To do this, we introduce a structural condition on $\Ucal$ that generalizes the classical smoothness condition from smoothed online learning. 
Recall that in smoothed online learning, the adversary is constrained to choose distributions that have bounded density ratio with respect to a fixed reference measure $\mu_0$; that is, there exists $\sigma>0$ such that for every $\mu\in\Ucal$ and measurable set $A$, $\mu(A) \leq \mu_0(A)/\sigma$. 
\emph{Generalized Smoothness} relaxes this linear relationship to allow for more general tolerance functions. 
That is, a distribution family $\Ucal$ is said to be $ (\rho, \mu_0) $-\emph{generalized} smooth if for every $\mu\in\Ucal$, we have 
\begin{align}
    \mu(A) \leq \rho(\mu_0(A)) \quad \text{for all measurable sets } A. 
\end{align}


This condition implies that the adversary cannot place significant mass on sets that are small under the base measure---but the relationship need not be linear as in classical smoothness. 
Our main result (\cref{thm:main}) shows that this condition is necessary for any form of learnability against every finite-VC class $\Fcal$.
In fact, we show a stronger result: generalized smoothness is necessary even to achieve asymptotic consistency against thresholds, the simplest non-trivial function classes.

To complete the characterization, we build upon recent algorithmic techniques for smoothed online learning \citep{shetty2024learning,haghtalab2024smoothed, block2024performance,blanchard2025agnostic} to show that generalized smoothness is sufficient to achieve VC-dimension-dependent regret bounds for every finite-VC class $\Fcal$, even when $\Ucal$ is not explicitly known to the learner (\cref{thm:regret_R_cover,thm:universal_realizable}).
These results show that the complex interplay of combinatorial structure and regret bound for distributionally constrained adversaries can be fully captured by a single structural condition on the distribution family $\Ucal$. Additionally, and perhaps surprisingly, this also shows that merely asking for asymptotic consistency on simple VC classes further implies non-asymptotic regret bounds. Note that asymptotic and non-asymptotic learnability are, in general, qualitatively different notions with distinct characterizations.

Our main technical contribution is establishing a connection between regret bound against threshold classes to a uniform notion of compactness of the distribution family $\Ucal$, which we capture via the existence of a continuous dominating submeasure (\cref{def:continuous_submeasure}) for $\Ucal$. 
Further, we show that the existence of such a continuous dominating submeasure is equivalent to generalized smoothness.  
The relation between continuity of submeasures and uniform domination by a measure (corresponding to generalized smoothness) is well-studied in measure theory, with celebrated results such as the \emph{Kalton-Roberts theorem} \citep{kalton1983uniformly} \footnote{Interestingly, this result is perhaps the first significant example of the use of expander graphs outside graph theory} and (Talagrand's resolution of) \emph{Maharam's problem} \citep{talagrand2008maharam}. 
In our setting, with the extra structure of our submeasure (arising as a supremum over a family of probability measures), we are able to establish a clean equivalence between continuity and uniform domination directly without relying on opaque application of this heavy machinery but developing a deeper connection to these classical results is an interesting direction for future work. 
 
We complement the characterization with algorithms (\cref{thm:universal_realizable,thm:regret_R_cover}) that achieve VC-dimension-dependent regret bounds when $\Ucal$ is generalized smooth, even when $\Ucal$ is unknown to the learner.
These can be viewed as the generalized smoothness analogues of the algorithms for smoothed online learning with unknown base measure from \cite{block2024performance,blanchard2025agnostic}. 
Finally, when $\Ucal$ is known to the learner, we provide refined regret bounds (\cref{thm:distrib_dependent_upper_bound,thm:distrib_dependent_lower_bound}) that depend on a combinatorial parameter of $\Ucal$ that we term the \emph{fragmentation number} (\cref{def:NU}). 
While having finite fragmentation numbers is equivalent to generalized smoothness, the precise dependence of regret bounds on fragmentation numbers provides a more refined understanding of the difficulty of learning under a given distribution family $\Ucal$, corresponding to the different rates in the \emph{known-base-measure} smoothed online learning setting \citep{haghtalab2024smoothed} and the \emph{unknown-base-measure} setting \citep{block2024performance,blanchard2025agnostic}.

In addition to the characterization of online learnability in the distribution constrained setting, generalized smoothness serves as a characterization of private learnability under the distributionally constrained settings.
This can be seen as the distributionally constrained analogue of the characterization of private learnability via the Littlestone dimension \citep{alon2019private}. 
We show that a distribution family $\Ucal$ admits private learnability with sample complexity depending only on the VC dimension of the hypothesis class if and only if $\Ucal$ is generalized smooth (\cref{thm:private_main_result,thm:private_sufficiency}). 


\section{Preliminaries}\label{sec:prelims}
Let $\Xcal$ be a separable measurable instance space and denote by $\Bcal$ the collection of measurable sets of $\Xcal$. Let $\Ycal=\{0,1\}$ denote the label space. 
Let $\Fcal \subseteq \Bcal$ be a family of measurable binary classifiers, associated with indicators over sets, and use the indicator loss $ \1\{ f(x) \neq y \}.$ 
While we focus on the case of binary classification and the indicator loss for concreteness, extending ideas proposed here to multiclass classification and regression is an interesting direction for future work.
The central combinatorial quantity in the binary classification setting is the VC dimension \citep{vapnik1971uniform}. 

\begin{definition}[VC dimension]
    The VC dimension of a function class $\Fcal:\Xcal\to\{0,1\}$ is defined as
    \[
        \mathrm{VC}(\Fcal) \;=\; \sup \left\{ d\in\Nbb : \exists x_1,\dots,x_d\in\Xcal \text{ s.t. } |\{(f(x_1),\dots,f(x_d)):f\in\Fcal\}| = 2^d \right\} \, .
    \]
    A set $\{x_1,\dots,x_d\}\subseteq\Xcal$ satisfying the above condition is said to be \emph{shattered} by $\Fcal$. 
\end{definition}


We next define the online learning protocol under distributional constraints.
Let $\Ucal$ be a family of probability distributions over $\Xcal$. 
Consider the following $T$-round online protocol. For each round $t=1,\dots,T$:
\begin{enumerate}
    \item The adversary selects a distribution $\mu_t\in\Ucal$, either \emph{obliviously} (the sequence $(\mu_t)_{t=1}^T$ is fixed in advance) or \emph{adaptively} based on the history.
    \item The learner chooses $f_t\in \Bcal $ (possibly randomized).
    \item A sample $x_t \sim \mu_t$ is drawn; then a label $y_t\in\Ycal$ is revealed. The learner incurs loss $\1\{ f_t(x_t) \neq y_t \}$ (and, in the full-information setting, observes $(x_t,y_t)$). Note that $\mu_t$ is not observed.
\end{enumerate}
The objective of the learner is to minimize regret against the comparator class $\Fcal$, that is, the difference between the cumulative loss of the learner and that of the best fixed function in $\Fcal$ in hindsight. Formally, define the regret after $T$ rounds for an algorithmn $\text{alg}$ as
\[
    \mathrm{Reg}_T(\text{alg};\Fcal,\Ucal) 
    \,=\, \Ebb \Bigg[\sum_{t=1}^T \1\{ f_t(x_t) \neq y_t \} 
    \, - \, \inf_{f\in\Fcal} \sum_{t=1}^T \1\{ f(x_t) \neq y_t \} \Bigg] \, .
\]

The setup descibed above generalizes both the classical adversarial online learning setting (when $\Ucal$ contains all distributions over $\Xcal$) and the (sequential analogue of the) standard i.i.d.\ statistical learning setting (when $\Ucal$ is a single fixed distribution).
We denote the best achievable regret against the worst-case $\Ucal$-constrained adversary by $\mathrm{Reg}_T(\Fcal,\Ucal)$. 
Note that when $\Ucal$ is a singleton, from the VC theory of statistical learning \citep{vapnik1971uniform}, we have $\mathrm{Reg}_T(\Fcal,\Ucal) \lesssim \sqrt{\mathrm{VC}(\Fcal) T}$. 
Futher, when $\Ucal$ contains all distributions over $\Xcal$, from the classical theory of adversarial online learning \citep{littlestone1988learning,ben2009agnostic}, we have that $\mathrm{Reg}_T(\Fcal,\Ucal) \lesssim \sqrt{\mathrm{LD}(\Fcal) T}$, where $\mathrm{LD}(\Fcal)$ is the Littlestone dimension of $\Fcal$ \citep{littlestone1988learning}\footnote{Since we do not directly work with the Littlestone dimension, we defer the definition to the appendix (\cref{def:littlestone})}. 
The key issue we aim to address is the fact that $\mathrm{LD} \gg \mathrm{VC} $ for most natural function classes, with $\mathrm{LD}$ typically being infinite.    
With this is mind, the aim of this work is to characterize conditions on $\Ucal$ that allow achieving regret bounds that depend on the VC dimension of $\Fcal$.  
In order to capture this, we introduce two notions of learnability for distribution families. 

\begin{definition}[Weak VC Learnability]
\label{def:weak_VC_learnability}
    A distribution class $\Ucal$ on $\Xcal$ is weakly VC learnable if for any function class $\Fcal$ with finite VC dimension, there exists an online algorithm $alg$ such that for any $\Ucal$-constrained adversary generating a data sequence $(x_t,y_t)_{t\geq 1}$ with the label $y_t$ consistent with some function in $\mathcal{F}$, we have
    \begin{equation*}
        \lim_{T\to\infty}\Ebb\sqb{ \frac{1}{T} \sum_{t=1}^T \1[\hat y_t\neq y_t] }=0.
    \end{equation*}
\end{definition}

Note that the definition above merely asks for \emph{asymptotic} consistency on realizable data: convergence rates are allowed to potentially depend on the adversary's behavior. 
Arguably, this is the weakest meaningful learning guarantee. To emphasize the difference between asymptotic consistency and finite-regret bounds, we note that any countable function class $\Fcal$---in particular, even with infinite VC dimension---admits asymptotically consistent learners even when the distribution class $\Ucal$ is fully unrestricted, i.e., even for adversarial data (this is for instance a direct consequence of \cite[Lemma 34]{hanneke2021learning}).
We can strengthen the notion of learnability by requiring \emph{non-asymptotic} regret bounds that depend on the VC dimension of the comparator class. 
We will refer to this as \emph{strong VC learnability}.

\begin{definition}[Strong VC Learnability]
\label{def:VC_learnability}
    A distribution class $\Ucal$ on $\Xcal$ is strongly VC learnable if for any function class $\Fcal$ with finite VC dimension, for any $\epsilon>0$ there exists $m_\Ucal(\epsilon,d)$ and an online algorithm $alg$ such that for any $\Ucal$-constrained adversary, and $T\geq m_\Ucal(\epsilon,d)$, 
    \begin{equation*}
        \mathrm{Reg}_T(\text{alg};\Fcal,\Ucal) \leq \epsilon T.
    \end{equation*}
\end{definition}

The key structural condition that captures learnability is \emph{generalized smoothness}. 

\begin{definition}[Generalized smoothed distribution class] \label{def:gen_smooth}
    Let $\Xcal$ be a measurable space. 
    Fix a measure $\mu_0$ on $\Xcal$ and a non-decreasing function $\rho:[0,1]\to \Rbb_+$ with $\lim_{\epsilon\to 0}\rho(\epsilon)=0$. 
    A family of distributions $\Ucal$ on $\Xcal$ is said to be $(\mu_0, \rho)$-generalized smoothed if for any distribution $\mu\in\Ucal$ and measurable set $A\subseteq\Xcal$,
    \begin{equation*}
        \mu(A) \leq \rho(\mu_0(A)).
    \end{equation*}
\end{definition}


One can see that generalized smoothness implies that the distributions in $\Ucal$ are absolutely continuous with respect to the base measure $\mu_0$, with $\rho$ acting as a uniform control on this continuity. 
Note that this definition generalizes classical smoothness, which corresponds to the special case $\rho(\epsilon)=\epsilon/\sigma$ for some constant $\sigma \leq 1$.
Further, this notion also generalizes smoothness notions based on $f$-divergences studied in \cite{block2023sample}.  
Last, the notion of generalized smoothness is closely related to the notion of coverage studied by \cite{chen2025coverage}. 
All these notions control the mass in the tails of density ratios between distributions in $\Ucal$ and a base measure $\mu_0$ and can be seen as different parameterizations of the same underlying concept. 
However, the notion of generalized smoothness is the most convenient for our purposes here and most directly addresses the underlying question of learnability.

\section{Learnability and Generalized Smoothness} \label{sec:gen_smoothness}
In this section, we characterize weak and strong VC learnability as defined in \cref{def:weak_VC_learnability,def:VC_learnability}. Specifically, we show that both are equivalent to generalized smoothness: in particular, weak VC learnability is---despite its asymptotic nature by definition---sufficient to recover non-asymptotic regret bounds. In fact, it turns out that weak learnability with respect to simple VC=1 classes is already sufficient for non-asymptotic regret bounds.

To formalize this result, we begin by introducing the family of \emph{generalized thresholds}, which are the prototypical example of VC one classes and conceptually, the simplest non-trivial function classes for learning.

\begin{definition}[Generalized thresholds]
\label{def:thresholds}
    A function class $\Fcal:\Xcal\to \{0,1\}$ is a \emph{generalized threshold} function class if there exists a total preorder\footnote{Preorders satisfy the same properties as orders except for antisymmetry: $\preceq$ is a preorder if it is reflexive and transitive.} $\preceq$ on $\Xcal$ such that $\Fcal \subseteq \{ \1[\cdot\preceq x_0]: x_0\in\Xcal\}$.
\end{definition}

These function classes naturally generalize the standard thresholds--- $\{ \1[\cdot \leq x_0]: x_0\in[0,1]\}$---on the interval $\Xcal=[0,1]$ to the case when one may use potentially different orderings than $\leq$ on the domain $[0,1]$. 
Note that the VC dimension of any generalized threshold class is at most $1$. 
Due to their simplicity, we use generalized thresholds as ``test classes'' for learnability.
We specialize the notion of weak learnability (\cref{def:weak_VC_learnability})  and strong learnability (\cref{def:VC_learnability}) to generalized thresholds and refer to these as \emph{weak threshold learnability} and \emph{strong threshold learnability}, respectively.

Surprisingly, we show in this section that weak threshold learnability is equivalent to generalized smoothness and hence sufficient to recover non-asymptotic regret-bounds for all function classes with finite VC dimension. In turn, this will imply the equivalence of weak and strong VC learnability to generalized smoothness.

As the first step of this proof, we study links between weak threshold learnability of a distribution class $\Ucal$ and properties of its envelope functional $\mu_\Ucal$ defined via
\begin{align*}
    \mu_\Ucal(A) := \sup_{\mu\in\Ucal} \mu(A), \quad A\in\Bcal.
\end{align*}
A key property will be whether $\mu_\Ucal$ is additionally a continuous submeasure, as defined below.

\begin{definition}[Continuous Submeasure] \label{def:continuous_submeasure}
    A function $\hat{\mu}: \Bcal \to [0,1]$ is a \emph{continuous submeasure} if:
    \begin{enumerate}
        \item \textbf{Monotonicity:} For any $A,B\in\Bcal$ such that $A\subseteq B$, $\hat{\mu}(A) \leq \hat{\mu}(B)$.
        \item \textbf{Sub-additivity:} For any $A,B\in\Bcal$, $\hat{\mu}(A\cup B) \leq \hat{\mu}(A) + \hat{\mu}(B)$.
        \item \textbf{Continuity from above:} For any decreasing sequence of measurable sets $(A_i)_{i\geq 1}$ with $\bigcap_{i\geq 1} A_i = \emptyset$, we have $\lim_{n\to\infty} \hat{\mu}(A_n) = 0$.
    \end{enumerate}
\end{definition}

Of course, the functional $\mu_\Ucal$ always satisfies monotonicity and sub-additivity for any distribution class $\Ucal$. More importantly, we show that $\mu_\Ucal$ must also satisfy the continuity from above property if the distribution class $\Ucal$ is weakly threshold learnable.

\begin{lemma}\label{lemma:threshold_to_continuous_submeasure}
    Let $\Ucal$ be a weakly threshold learnable distribution class on $\Xcal$. Then, $\mu_\Ucal$ is a continuous submeasure.
\end{lemma}

As a brief overview of the proof, if $\mu_\Ucal$ fails to be continuous, then there exist infinitely-many (countably) disjoint sets each carrying at least $\epsilon$ mass under some distribution in $\Ucal$ for some parameter $\epsilon>0$. We index these sets by rationals and construct a threshold class where each function $f_r$ is the indicator of the union of sets indexed by rationals at most $r$. By construction, the adversary can place $\epsilon$ mass under each of these sets, which allows the adversary to simulate an unconstrained adversary on the threshold function class. This yields an expected lower bound $\Omega(\epsilon T)$ on the number of mistakes following classical lower bound arguments for adversarial online learning \citep{littlestone1988learning}, up to minor modifications to ensure realizability. Essentially, the adversary hides the true threshold $r^\star$ through a binary search.
We defer the full proof to \cref{app:proof_converse}.


The main step of the proof is to show that if $\mu_\Ucal$ is a continuous submeasure, then the distribution class $\Ucal$ corresponds to generalized smoothed classes.

\begin{lemma}\label{lemma:continuous_submeasure_to_smooth}
    Let $\Ucal$ be a family of distributions on $\Xcal$ such that $\mu_\Ucal$ is a continuous submeasure. Then, $\Ucal$ is $(\mu_0,\rho)$-generalized smooth for some distribution $\mu_0$ on $\Xcal$ and a non-decreasing function $\rho:[0,1]\to \Rbb_+$ with $\lim_{\epsilon\to 0}\rho(\epsilon)=0$.
\end{lemma}

The proof is constructive. At the high-level, if $\mu_\Ucal$ is continuous then for any scale $\epsilon>0$ then we can construct a finite cover $\{\mu_1,\ldots,\mu_N\}\in\Ucal$ of the distribution class $\Ucal$ for any measurable set $A\in\Bcal$ on which some distribution $\mu\in\Ucal$ places $\epsilon$ mass. Taking an appropriate mixture of these finite covers for decreasing scales $\epsilon\to 0$ then yields a distribution $\mu_0$ with the desired generalized-smoothness property on $\Ucal$.

\vspace{3pt}

\begin{proof}
    Fix $\Ucal$ satisfying the assumptions and any $\epsilon>0$. We iteratively construct a sequence of distributions $(\mu_i^\epsilon)_{i\geq 1}$ together with measurable sets $(A_i^\epsilon)_{i\geq 1}$ as follows. Suppose we have constructed $\mu_j^\epsilon,A_j^\epsilon$ for all $j<i$ for some $i\geq 1$. If there exist a distribution $\mu\in\Ucal$ and a measurable set $A$ such that
    \begin{equation*}
        \mu(A) \geq \frac{2}{\epsilon} \sum_{j<i} 2^{j-i}\mu_j^\epsilon(A) \quad \text{and} \quad \mu(A)\geq \epsilon,
    \end{equation*}
    we fix $\mu_i^\epsilon,A_i^\epsilon$ satisfying this property. Otherwise, we end the recursive construction.

    Suppose by contradiction that this construction does not terminate. Define $B_i^\epsilon:=A_i^\epsilon\setminus \bigcup_{j>i}A_j^\epsilon$ for all $i\geq 1$. Then,
    \begin{equation*}
        \mu_\Ucal(B_i^\epsilon) \geq \mu_i^\epsilon(B_i^\epsilon)\geq \mu_i^\epsilon(A_i^\epsilon) - \sum_{j>i} \mu_i^\epsilon(A_j^\epsilon) \geq \epsilon - \sum_{j>i} \frac{\epsilon}{2^{j-i+1}} \mu_j^\epsilon(A_j^\epsilon) \geq \frac{\epsilon}{2}.
    \end{equation*}
    Since the sets $B_i^\epsilon$ are disjoint, this contradicts the continuous submeasure assumption on $\mu_\Ucal$. Specifically, we have $\mu(\bigcap_{j\geq i}B_i)\geq \epsilon/2$ while $\bigcap_{j\geq i}B_j\downarrow\emptyset$. As a result, the construction terminates. Let $N_\epsilon$ be the last constructed index. Then, by construction, the finite measure $\mu^\epsilon:=\frac{2}{\epsilon}\sum_{j\leq N_\epsilon}\mu_j^\epsilon$ satisfies for any measurable set $A\subseteq\Xcal$,
    \begin{equation}\label{eq:property_sup}
        \mu_\Ucal(A) = \sup_{\mu\in\Ucal} \mu(A) \leq \epsilon + \mu^\epsilon(A).
    \end{equation}

    Now fix any positive decreasing sequence $(\epsilon_i)_{i\geq 1}$ in $[0,1]$ with $\lim_{i\to\infty}\epsilon_i=0$. For each $i\geq 1$, since $\mu^{\epsilon_i}$ is a finite measure, we can write it as $\mu^{\epsilon_i} = C_i\nu_i$ where $\nu_i$ is a probability measure and $C_i=\mu^{\epsilon_i}(\Xcal)\geq 0$. Without loss of generality, we may assume that $(C_i)_{i\geq 0}$ is non-decreasing while preserving \cref{eq:property_sup}. In particular, the sequence formed by $\delta_i:=\frac{\epsilon_i}{2^iC_i}$ is non-increasing in $i\geq 1$. We then define the probability measure
    \begin{equation*}
        \mu_0:=\sum_{i\geq 1} 2^{-i} \nu_i,
    \end{equation*}
    and we fix a non-decreasing function $\rho:[0,1]\to \Rbb_+$ by setting $\rho(z)=2\epsilon_i$ for all $z\in(\delta_{i+1},\delta_i]$ for all $i\geq 1$, $\rho(0)=0$ and $\rho(z)=2$ for $z\in(\delta_1,1]$.
    
    Now fix a measurable set $A\subseteq\Xcal$ and denote $z=\mu_0(A)$. If $z>\delta_1$ then we immediately have $\mu_\Ucal(A)\leq 1=\rho(\mu_0(A))$. Next, if $z\in(\delta_{i+1},\delta_i]$ for some $i\geq 1$, from \cref{eq:property_sup} we have
    \begin{equation*}
        \mu_\Ucal(A) \leq \epsilon_i + \mu^{\epsilon_i}(A) = \epsilon_i + C_i \nu_i(A) \overset{(i)}{\leq} \epsilon_i + C_i2^i\mu_0(A) \leq 2\epsilon_i =\rho(\mu_0(A)).
    \end{equation*}
    In $(i)$ we used the definition of $\mu_0$.
    Finally, if $z=0$, then similarly as above, \cref{eq:property_sup} implies that $\mu_\Ucal(A)\leq \epsilon_i$ for all $i\geq 1$ and hence $\mu_\Ucal(A)=0=\rho(\mu_0(A))$. This ends the proof.
\end{proof}

Combining \cref{lemma:threshold_to_continuous_submeasure,lemma:continuous_submeasure_to_smooth} 
directly shows that any weakly threshold learnable family of distributions $\Ucal$ is a generalized smoothed distribution class. It is further known that generalized smoothness classes are learnable for all VC classes \citep{blanchard2025distributionally}. This yields the following main characterization theorem. 

\begin{theorem}\label{thm:main}
    Let $\Ucal$ be a family of distributions on $\Xcal$. The following are equivalent:
    \begin{enumerate}
        \item $\Ucal$ is weakly threshold learnable.
        \item $\Ucal$ is strongly VC learnable.
        \item $\Ucal$ is a generalized smoothed distribution class.
    \end{enumerate}
\end{theorem}

\begin{proof} We proved \textbf{(1) $\Rightarrow$ (3)} in \cref{lemma:threshold_to_continuous_submeasure,lemma:continuous_submeasure_to_smooth}. Specifically, if $\Ucal$ is weakly threshold-learnable, then \cref{lemma:threshold_to_continuous_submeasure} shows that $\mu_\Ucal$ is a continuous submeasure and hence \cref{lemma:continuous_submeasure_to_smooth} exactly constructs a reference measure $\mu_0$ and a tolerance function $\rho$ for which $\Ucal$ is generalized smooth. \textbf{(2) $\Rightarrow$ (1)} is immediate. \textbf{(3) $\Rightarrow$ (2):} The learnability of all VC classes under generalized smoothed distribution classes was already observed in previous works, e.g. \cite[Corollary 10]{blanchard2025distributionally}. We will further give concrete algorithms---\text{ERM} and \textsc{R-Cover}---that achieves VC-dimension-dependent regret bounds for any generalized smoothed distribution classes in the realizable and agnostic setting respectively within \cref{thm:universal_realizable,thm:regret_R_cover}.
\end{proof}

Conceptually, \cref{thm:main} shows that any form of learnability on generalized thresholds---while being very simple VC=1 function classes---already forces strong structural properties on the distribution class $\Ucal$, as captured by generalized smoothness. Perhaps more surprisingly, this shows that \emph{asymptotic} learnability guarantees on generalized thresholds imply \emph{non-asymptotic} regret guarantees on all VC classes. 

As a remark, this property may not hold for simpler classes than generalized thresholds. For instance, note that any countable function class $\Fcal$ is weakly learnable in the sense that there exists an algorithm which ensures sublinear excess regret compared to any fixed function within $\Fcal$ even for unconstrained adversaries, e.g. see \cite[Lemma 34]{hanneke2021learning}. Hence, $\Ucal=\{\text{all distributions on }\Xcal\}$ is weakly learnable for any countable function class, but it is not learnable for any function class with infinite Littlestone dimension---including threshold function classes.


\subsection{Uniform Covers} 

Another perspective on generalized smoothness is through the lens of uniform covers.

\begin{definition}[Uniform Cover] \label{def:unif_covers}
    A distribution class $\Ucal$ is said to admit \emph{uniform covers} for a function class $\Fcal$ if for any $\epsilon>0$, there exists a finite set of functions $\Fcal_\epsilon\subseteq \Fcal$ such that for any $\mu\in\Ucal$ and $f\in\Fcal$, there exists $\hat f\in\Fcal_\epsilon$ such that $\mu(\{x: f(x)\neq \hat f(x)\})\leq \epsilon$. 
\end{definition}

It is easy to see that generalized smoothness implies existence of uniform covers for all VC classes. Indeed, if $\Ucal$ is generalized smooth with respect to some reference measure $\mu_0$ and tolerance function $\rho$, then for any $\epsilon>0$, using a covering for $\Fcal$ under $\mu_0$ at scale $\delta=\rho^{-1}(\epsilon)$ yields a uniform cover at scale $\epsilon$ for all distributions in $\Ucal$. We record this formally as \cref{lem:unif_cover}. 

A distribution class admitting uniform covers for a class $\Fcal$ is also sufficient for achieving low regret if the adversary is \emph{oblivious} and the distribution class $\Ucal$ is known. 
Indeed, as observed in \cite{haghtalab2018foundation}, a standard Hedge algorithm run on the uniform cover $\Fcal_\epsilon$ ensures regret at most $\epsilon T + O(\sqrt{T\log|\Fcal_\epsilon|})$ against any oblivious $\Ucal$-constrained adversary. 
But, to the best of our knowledge, there is no known direct way to extend this to the adaptive adversaries setting or the case when $\Ucal$ is unknown, without using the machinery of the coupling lemma introduced by \cite{haghtalab2024smoothed} (which requires stronger assumption on the distribution class than just uniform covers). 

A modification of our proof of \cref{thm:main} actually shows that the existence of uniform covers for all VC classes is indeed equivalent to generalized smoothness.  

\begin{proposition}\label{prop:uniform_covers_equal_gen_smoothness}
    Let $\Ucal$ be a distribution class. Then $\Ucal$ has uniform covers for all VC classes (equivalently, all generalized threshold classes) if and only if it is generalized smooth.
\end{proposition}

Thus, our result can be seen as showing that if we insist on VC dependent regret bounds then uniform cover is also a necessary and sufficient condition for learnability.

\section{Universal Algorithms}\label{sec:universal_algs}
In the following two sections, we provide algorithms that achieve low regret under generalized smoothed distribution classes. As a reminder, as per the definitions of learnability in \cref{def:VC_learnability,def:weak_VC_learnability}, a distribution class $\Ucal$ is learnable if there exists a corresponding algorithm which achieves the desired learning guarantee. Importantly, the considered algorithm a priori requires knowledge of the distribution class (and the function class $\Fcal$, which is unavoidable). 
In this section, we instead focus on algorithms that do not require this prior knowledge on $\Ucal$, yet still achieve learning under VC classes (or generalized thresholds) and any distribution class $\Ucal$ for which this is achievable---equivalently, any generalized smoothed distribution class $\Ucal$, by \cref{thm:main}. We refer to such adaptive algorithms as \emph{optimistically universal}, following the corresponding literature on ``learning whenever possible'' \citep{hanneke2021learning,blanchard2022universal,blanchard2023universal}. Note that their existence is not a priori guaranteed.\footnote{For instance, no optimistically universal algorithm exist for inductive supervised learning \citep{hanneke2021learning}, or online adversarial contextual bandits \citep{blanchard2023adversarial}.}

To obtain universal algorithms with low regret under generalized smoothness, we leverage known results for smooth adversaries, which correspond to generalized smoothed classes with linear tolerance function $\rho(\epsilon)=\epsilon/\sigma$. Specifically, for smooth adversaries, Empirical Risk Minimization (ERM) already achieves sublinear regret for VC classes in the realizable setting \citep{block2024performance}, while \textsc{R-Cover} \citep{blanchard2025agnostic}---an algorithm based on recursive empirical covers---is known to achieve sublinear regret for VC classes in the agnostic setting.

\paragraph{Coupling reduction lemma.}
We then lift these guarantees to generalized smooth adversaries using a simple coupling reduction to smooth adversaries. Fix a generalized smoothed class $\Ucal$ with base measure $\mu_0$ and tolerance $\rho$. Recall that by definition, for any $\epsilon$, any distribution places at most mass $\rho(\epsilon)$ on sets of $\mu_0$-measure $\epsilon$. Hence, distributions in $\Ucal$ are intuitively $\epsilon/\rho(\epsilon)$-smooth with respect to $\mu_0$ up to a mass $\rho(\epsilon)$. More formally, we can couple the sequence generated by a $\Ucal$-constrained adversary with a $\Ocal(\epsilon/\rho(\epsilon))$-smooth sequence which agree except on a fraction $\Ocal(\rho(\epsilon))$ of the samples for which we replace the original sample with a dummy variable $x_\emptyset$, as detailed below. This argument will allows us to import regret bounds from the smoothed online learning literature to the case of generalized smoothed adversaries. 

For convenience, we say that a tolerance function $\rho:[0,1]\to\Rbb_+$ is \emph{well-behaved} if it is non-decreasing, $\lim_{\epsilon\to 0}\rho(\epsilon)$, and $\epsilon\in(0,1]\mapsto\rho(\epsilon)/\epsilon$ is non-increasing.

\begin{lemma}\label{lemma:coupling}
    Let $\Ucal$ be a generalized smoothed distribution class on $\Xcal$ with a well-behaved tolerance function $\rho:[0,1]\to\Rbb_+$ and a horizon $T\geq 1$. For convenience, fix a dummy new variable $x_\emptyset\notin\Xcal$. Consider any $\Ucal$-constrained adaptive adversary to generate samples $x_1,\ldots,x_T$.

    For any $\epsilon>0$, there exists a coupling of these samples with a sequence $\tilde x_1,\ldots,\tilde x_T\in\Xcal\cup\{x_\emptyset\}$ which is generated from a $\frac{\epsilon}{2\rho(\epsilon)}$-smooth adversary on $\Xcal\cup\{x_\emptyset\}$ such that for all $t\in[T]$, $\tilde x_t\in\{x_t,x_\emptyset\}$ and
    \begin{equation*}
        \Pbb[\tilde x_t=x_\emptyset\mid \tilde x_{<t},x_{<t}] \leq \rho(\epsilon).
    \end{equation*}
\end{lemma}

As a remark, \cref{lemma:coupling} additionally assumes that the tolerance function $\rho$ is such that $\rho(\epsilon)/\epsilon$ is non-increasing, which is a mild technical regularity condition. In fact, this property is automatic if distributions in $\Ucal$ are non-atomic (see \cref{prop:non_atomic_concavity} for a formal statement and proof). More generally, note that this is without loss of generality: any generalized smoothed distribution class $\Ucal$ admits a concave tolerance function $\rho$ with $\rho(0)=0$, up to replacing its original tolerance function with its concave envelope $\tilde\rho$ (the smallest concave function $\tilde \rho$ with $\tilde\rho(x)\geq \rho(x)$ for all $x\geq 0$). In most scenarios, this will not change the achieved rates.

\paragraph{Realizable setting.}
We start with the realizable setting, where there exists a function $f^\star\in\Fcal$ fixed \emph{a priori}, and such that $y_t=f^\star(x_t)$ for all $t\geq 1$. 
In this setting, we show that ERM universally learns VC function classes under any threshold learnable distribution class $\Ucal$.

\begin{theorem}\label{thm:universal_realizable}
    Let $\Fcal:\Xcal\to\{0,1\}$ be a function class with finite VC dimension and let $\Ucal$ be a generalized smoothed distribution class on $\Xcal$ with a well-behaved tolerance function $\rho:[0,1]\to\Rbb_+$. Then, for any realizable $\Ucal$-constrained adversary, \textsc{ERM} makes predictions $\hat y_t$ such that
    \begin{equation*}
        \mathrm{Reg}_T(\text{ERM}) = \Ebb\sqb{\sum_{t=1}^T \ell_t(\hat y_t)} \leq \inf_{\epsilon>0} \left\{ C\sqrt{\frac{\rho(\epsilon)}{\epsilon} dT  \ln^5 T} + \rho(\epsilon)T\right\}.
    \end{equation*}
\end{theorem}

The proof essentially combines the coupling reduction \cref{lemma:coupling} with known regret bounds for ERM in the realizable setting. Naturally, for any $\epsilon>0$, compared the bound for the $\Ocal(\epsilon/\rho(\epsilon))$-smooth adversaries from \cite[Theorem 8]{block2024performance}, the regret bound from \cref{thm:universal_realizable} includes an additive term $\rho(\epsilon)T$ corresponding to the expected number of disagreements between the original and smoothed sequence constructed in \cref{lemma:coupling}.

\paragraph{Agnostic setting.}
Next, we turn to the general agnostic setting in which no assumptions are made on the labels $y_t$. We show that the algorithm \textsc{R-Cover} from \cite{blanchard2025agnostic} also universally learns VC classes under any threshold learnable distribution class $\Ucal$.

\begin{theorem}\label{thm:regret_R_cover}
    Let $\Fcal:\Xcal\to\{0,1\}$ be a function class with finite VC dimension and let $\Ucal$ be a generalized smoothed distribution class on $\Xcal$ with a well-behaved tolerance function $\rho:[0,1]\to[0,1]$. Then, for any adaptive adversary, $\textsc{R-Cover}$ makes predictions $\hat y_t$ such that
    \begin{equation*}
        \Ebb\sqb{\sum_{t=1}^T \ell_t(\hat y_t) - \inf_{f\in\Fcal} \sum_{t=1}^T \ell_t(f(x_t) ) } \leq  C\ln^{3/2} T \cdot \inf_{\epsilon>0}\left\{\sqrt {\frac{\rho(\epsilon)}{\epsilon}dT}\cdot \ln T +\rho(\epsilon) T\sqrt d\right\},
    \end{equation*}
    for some universal constant $C>0$.
\end{theorem}

Due to the recursive nature of \textsc{R-Cover}, the proof requires additional care to handle the disagreements between the original sequence $x_1,\ldots,x_T$ and the smoothed sequence $\tilde x_1,\ldots,\tilde x_T$ from \cref{lemma:coupling}. Specifically, previous samples are used recursively to construct covers, hence disagreements between the two sequences may amplify with the depth of the construction. This increases the additive disagreement term in \cref{thm:regret_R_cover} by a factor $\sqrt d \cdot \text{polylog } T$, compared to the ERM regret bound in \cref{thm:universal_realizable}.

\section{Distribution-class-dependent regret bounds} \label{sec:dist_dep_bounds}
The previous sections established that generalized smoothness characterizes VC-dependent learnability and that universal algorithms exist which achieve low regret without knowing the specific distribution family $\Ucal$. A natural follow-up question is: \emph{can we do better when $\Ucal$ is known?}

When the distribution family $\Ucal$ is known to the learner, one can often obtain sharper regret bounds by exploiting additional quantitative structure of $\Ucal$. In this section, we develop a simple combinatorial parameterization of distribution families and connect it to upper and lower regret bounds. The key insight is that the ``complexity'' of a distribution family can be measured by how fragmented mass can be: if $\Ucal$ allows mass to be spread across many disjoint regions, learning is harder. This can be formalized by the $\epsilon$-\emph{fragmentation number} of the distribution family $\Ucal$, which intuitively measures how ``fragmented'' mass can be under distributions in $\Ucal$ at scale $\epsilon\in(0,1]$.

\begin{definition}[Fragmentation number]
\label{def:NU} 
For any distribution class $\Ucal$ and scale $\epsilon\in(0,1]$, the $\epsilon$-fragmentation number of $\Ucal$ is defined via
    \begin{equation*}
    N_\Ucal(\epsilon) := \max\{k\in\Nbb: \exists \text{ disjoint }A_1,\ldots,A_k\in\Bcal, \forall i\in[k], \mu_\Ucal(A_i)\geq \epsilon\}.
\end{equation*}
\end{definition}

As a useful remark, $\mu_\Ucal$ is a continuous submeasure if and only if $N_\Ucal(\epsilon)<\infty$ for all $\epsilon>0$ (see the beginning of the proof of \cref{lemma:threshold_to_continuous_submeasure} for the proof).

In words, the $\epsilon$-fragmentation number is maximum number of disjoint regions that can each carry at least $\epsilon$ mass under some $\mu\in\Ucal$. Equivalently, it counts disjoint regions on which a $\Ucal$-constrained adversary may place $\epsilon\in(0,1]$ mass. In particular, a $\Ucal$-constrained adversary may emulate, for a fraction $\epsilon$ of the rounds, an adversary which selects the region on which the next sample lies, among these disjoint regions. This observation leads to the following regret lower bound for $\Ucal$-constrained adversaries.

\begin{theorem}\label{thm:distrib_dependent_lower_bound}
    Fix a family of distributions $\Ucal$ on $\Xcal$ and $T\geq 1$. For any $\epsilon> 0$ and $d\geq 1$ such that $3d\leq N_\Ucal(\epsilon)$. Then, there exists a function class $\Fcal:\Xcal\to\{0,1\}$ of VC dimension at most $d$ such that for any algorithm $\text{alg}$,
    \begin{equation*}
        \textrm{Reg}_T(\text{alg};\Fcal,\Ucal) \gtrsim \sqrt{\epsilon d T \ln \frac{N_\Ucal(\epsilon)}{d}}\wedge (\epsilon T).
    \end{equation*}
    for some oblivious $\Ucal$-constrained adversary.
\end{theorem}

We next turn to regret upper bounds. We first give a quantitative version of \cref{lemma:continuous_submeasure_to_smooth} using fragmentation numbers. This essentially bounds the continuity modulus of the envelope functional $\mu_\Ucal$ as a function of fragmentation numbers.

\begin{lemma}\label{lemma:continuous_submeasure_to_smooth_quantitative}
    Let $\Ucal$ be a family of distributions and $\epsilon\in(0,1]$ such that $N_\Ucal(\epsilon)<\infty$. Then, there exists a distribution $\mu_\epsilon$ on $\Xcal$ such that for any measurable set $A\subseteq\Xcal$,
    \begin{equation*}
        \mu_\epsilon(A)\leq \frac{\epsilon}{N_\Ucal(\epsilon)^2} \implies \mu_\Ucal(A)\leq 2\epsilon.
    \end{equation*}
\end{lemma}

The proof is constructive and largely similar to that of \cref{lemma:continuous_submeasure_to_smooth}. Intuitively, \cref{lemma:continuous_submeasure_to_smooth_quantitative} shows that at the scale of events of mass $\Omega(\epsilon)$, the distribution class $\Ucal$ is essentially $\Ocal(N_\Ucal(\epsilon)^2)$-smooth with respect to a fixed base distribution $\mu_\epsilon$.
With this property at hand, we can run a simple Hedge algorithm over an appropriate cover of $\Fcal$, following standard agruments for smoothed adversaries \citep{haghtalab2024smoothed,block2022smoothed}, to achieve the following regret upper bound.

\begin{theorem}\label{thm:distrib_dependent_upper_bound}
    Let $\Ucal$ be a family of distributions on $\Xcal$ and $T\geq 1$. For any function class $\Fcal:\Xcal\to\{0,1\}$ of VC dimension $d\geq 1$, there exists an algorithm $\text{alg}$ such that
    \begin{equation*}
        \text{Reg}_T(\text{alg};\Fcal,\Ucal) \lesssim \inf_{\epsilon>0} \{ \sqrt{dT\ln (TN_\Ucal(\epsilon))} + \epsilon T \}.
    \end{equation*}
\end{theorem}

As an important remark, both upper and lower bounds depend only on the logarithm of fragmentation numbers $N_\Ucal(\epsilon)$ for $\epsilon\in(0,1]$. In comparison, the bounds from \textsc{ERM} and \textsc{R-Cover} qualitatively depend \emph{polynomially} on the fragmentation numbers. Specifically, if $N_\Ucal(\epsilon)<\infty$ for all $\epsilon>0$, \cref{lemma:continuous_submeasure_to_smooth_quantitative} suggests that $\Ucal$ is generalized smooth with a function $\rho$ such that $\rho(\epsilon/N_\Ucal(\epsilon)^2)=2\epsilon$ for $\epsilon\in(0,1]$. Provided that $\rho$ is additionally concave, \cref{thm:universal_realizable} yields the following regret bound for \textsc{ERM}:
\begin{equation*}
    \text{Reg}_T(\textsc{ERM};\Fcal,\Ucal) \lesssim  \inf_{\epsilon>0} \{N_\Ucal(\epsilon)\sqrt{dT\ln^5 T} + \epsilon T\},
\end{equation*}
for function classes $\Fcal$ of VC dimension $d$. \cref{thm:regret_R_cover} yields similar regret bounds for \textsc{R-Cover}. Of course, this regret bound is looser than \cref{thm:distrib_dependent_upper_bound} but has the advantage of being achieved by a universal algorithm that does not require knowledge of the distribution class $\Ucal$.

\section{Differential Privacy} \label{sec:differential_privacy}

In this section, we will extend the insights from our characterization of distributionally constrained online learning to the setting of differentially private learning.
Recall that differential privacy \citep{dwork2006calibrating} is a rigorous framework for ensuring that learning algorithms do not leak sensitive information about individual data points.
Though differential privacy can be defined is broad generality, we will focus on the setting of learning binary classifiers and present the definition specialized to this setting.

\begin{definition}[Differential Private Learning]
    A randomized algorithm $\mathcal{A}:(\Xcal\times\{0,1\})^m\to \left\{ 0,1 \right\}^{\mathcal{X}} $ is $(\alpha,\beta)$-differentially private if for any pair of datasets $S,S'\in\Xcal^m$ differing in a single entry, and for any measurable set $E\subseteq \Hcal$,
    \begin{equation}
        \Pbb(\mathcal{A}(S)\in E) \leq e^\alpha \Pbb(\mathcal{A}(S')\in E) + \beta.
    \end{equation}
    Further, for a distribution class $\Ucal$ on $\Xcal$, a function class $\Fcal:\Xcal\to\{0,1\}$ is said to be $(\alpha,\beta)$-privately $(\epsilon,\delta)$-accurately learnable under $\Ucal$ with sample complexity $m$ if there exists an $(\alpha,\beta)$-differentially private algorithm $\mathcal{A}:\Xcal^m\to \left\{ 0,1 \right\}^{\mathcal{X}} $ such that for any distribution $D$ over $\Xcal\times\{0,1\}$ with marginal on $\Xcal$ in $\Ucal$, we have
    \begin{equation}
        \Pbb_{S\sim D^m} \left( L_D(\mathcal{A}(S)) \leq \inf_{f\in\Fcal} L_D(f) + \epsilon \right) \geq 1-\delta
    \end{equation}
    where $L_D(f) = \Pbb_{(x,y)\sim D}(f(x)\neq y)$.
\end{definition}

With this definition at hand, we can now state our characterization of private learnability under distribution constraints. 
This characterization mirrors our main result for online learning: generalized smoothness is necessary for private learnability with VC-dependent sample complexity.

\begin{theorem}\label{thm:private_main_result}
    There is a constant $c_0>0$ such that the following holds. Let $\Ucal$ be a distribution class on $\Xcal$ such that for any generalized threshold class $\Fcal$ and any $\epsilon>0$, there exists $m(\epsilon,\Fcal)<\infty$ such that for $\Fcal$ is $(0.1,\frac{c_0}{m(\epsilon,\Fcal)^2\log m(\epsilon,\Fcal)})$-privately $(\epsilon,1/32)$-accurately learnable under $\Ucal$ with sample complexity $m(\epsilon,\Fcal)$. Then, $\Ucal$ is a generalized smooth distribution class.
\end{theorem}

As a remark, we believe that a stronger version of this statement may be possible, by showing that generalized smoothness also characterizes private \emph{asymptotic consistency}, for which accuracy is only required asymptotically for each fixed data distribution.

We complement this result by showing that differentially private learnability with VC-dependent sample complexity is achievable under any generalized smooth distribution class. 
In fact, this is a simple consequence of the fact that under generalized smoothness VC class has uniform covers (\cref{def:unif_covers}) and the fact that for finite classes the exponential mechanism \citep{mcsherry2007mechanism} provides private learnability with logarithmic dependence on the size of the class.

\begin{theorem}[Private Learnability under Generalized Smoothness] \label{thm:private_sufficiency}
    Let $\Ucal$ be a $(\rho, \mu_0)$ generalized smooth distribution class on $\Xcal$. Then, for any function class $\Fcal:\Xcal\to\{0,1\}$ of VC dimension $d\geq 1$, any $\epsilon,\delta\in(0,1)$, and any $\alpha\in(0,1)$, there exists an $(\alpha,0)$-differentially private algorithm that $(\epsilon,\delta)$-accurately learns $\Fcal$ under $\Ucal$ with sample complexity 
    \begin{equation*}
        m \lesssim 
        \frac{d \log(1/\epsilon) + \log(1/\delta)}{\epsilon^2} + \frac{d \log( \rho^{-1}(\epsilon ) ) }{\alpha \epsilon}.
    \end{equation*}
\end{theorem}

This result can be seen as being related to the line of work on private learning with unlabelled public data \citep{alon2019limits,bassily2020private, block2024oracleefficientdifferentiallyprivatelearning, block2025smalllossboundsonline} since unlabelled public data from the base distribution suffices to get the improved sample complexity guarantees. 
Further, using techniques from \cite{haghtalab2020smoothed} similar results can be obtained for query release and data release tasks under generalized smoothness assumptions.
It is an interesting research direction to understand whether oracle-efficient private learning algorithms \citep{block2024oracleefficientdifferentiallyprivatelearning,block2025smalllossboundsonline} can be obtained under generalized smoothness assumptions.

\section{Discussion}\label{sec:discussion}

We have established that generalized smoothness exactly characterizes when VC-dimension-dependent regret bounds are achievable against distributionally constrained adversaries. This characterization bridges the gap between the classical $\sigma$-smooth model and the fully adversarial setting, providing a complete picture of learnability under distribution shifts. 
Further this characterization extends to the setting of differentially private learning, showing that generalized smoothness is necessary and sufficient for private learnability with VC-dependent sample complexity.

Our work fits into the broader program of understanding ``beyond worst-case'' models in online learning and private learning. 
The characterization suggests that generalized smoothness is a fundamental structural property, much like how VC dimension and Littlestone dimension are fundamental for i.i.d.\ and adversarial learning respectively. 
We hope this perspective inspires further work on understanding the landscape of learnability under various distributional assumptions.

\acks{We acknowledge support from ARO through award W911NF-21-1-0328, Simons Foundation, and the NSF through awards DMS-2031883 and PHY-2019786, the DARPA AIQ program, and AFOSR FA9550-25-1-0375. Abhishek Shetty was supported by an NSF FODSI fellowship.}

\newpage

\bibliographystyle{alpha}
\bibliography{refs}

\newpage
\appendix

\crefalias{section}{appendix}

\section{Proofs from \cref{sec:gen_smoothness}} \label{app:proof_converse}

We start by proving that for threshold learnable distribution classes $\Ucal$, their envelope measure $\mu_\Ucal$ is a continuous submeasure.

\vspace{3pt}

\begin{proof}{\textbf{of \cref{lemma:threshold_to_continuous_submeasure}}}
    The first two properties of continuous submeasures are immediate to check by definition of $\mu_\Ucal:=\sup_{\mu\in\Ucal}\mu$. We focus on proving the continuity property.

    Suppose by contradiction that $\mu_\Ucal$ is not continuous and fix a sequence of measurable sets $B_i\downarrow\emptyset$ and $\epsilon>0$ such that for all $i\geq 1$, $\mu_\Ucal(B_i)\geq 2\epsilon$. From these decreasing sets, we construct by induction a sequence of disjoint measurable sets $(A_i)_{i\geq 1}$ such that for $i\geq 1$ we have $\mu_\Ucal(A_i)\geq \epsilon$, as follows. Let $j_1=1$. For a given $i\geq 1$, let $\mu\in\Ucal$ with $\mu(B_{j_i})> \frac{3}{2}\epsilon$, which exists by construction of $(B_i)_{i\geq 1}$. Since $B_j\downarrow\emptyset$, there exists $j_{i+1}>j_i$ such that $\mu(B_{j_i}\setminus B_{j_{i+1}})\geq \epsilon$. We then pose $A_i:= B_{j_i}\setminus B_{j_{i+1}}$; in particular, $\mu_\Ucal(A_i)\geq \mu(A_i)\geq \epsilon$. By construction, $(A_i)_{i\geq 1}$ are all disjoint, which ends the construction.

    Fix any bijection $\phi:\Qbb\to\Nbb$ between positive integers and rationals. For any rational $r\in\Rbb$, we define the function $f_r:\Xcal\to\{0,1\}$ to be the indicator of the set $\bigcup_{y\leq x, y\in\Qbb}A_{\phi(y)}$.
    We then consider the function class which collects of all such functions: $\Fcal:=\{f_r: r\in\Rbb\}$. 
    By construction, we can check that $\Fcal$ is a generalized threshold function class since for any $r\leq s\in\Rbb$ one has $f_r\leq f_s$. For convenience, in the rest of the proof we use the following abuse of notation. For any interval $I\subseteq\Rbb$ and $x\in\Xcal$ we write $x\in I$ if and only if $x\in\bigcup_{r\in I\cap \Qbb}A_{\phi(r)}$.

    We now show that $\Ucal$ is not weakly learnable for $\Fcal$. The corresponding construction follows classical arguments.
    Specifically, we fix an online learner {\sf alg} and construct an instance on which they incur linear regret. To do so, we recursively a sequence of rationals $(a_t,b_t,r_t)_{t\geq 1}$ such that $a_t<b_t$ for all $t\geq 1$, and corresponding distributions $(\mu_t)_{t\geq 1}$. By default, initially we can pose $a_0=0$ and $b_0=1$). Fix $t\geq 1$ and suppose $a_{t-1}<b_{t-1}$ have been constructed. We first pose
    \begin{equation*}
        c_t:=\frac{a_{t-1}+b_{t-1}}{2}.
    \end{equation*}
    By construction, there exists $\mu_t\in\Ucal$ such that
    \begin{equation}\label{eq:lb_mass_A_c_t}
        \mu_t(A_{\phi(c_t)}) \geq \frac{\mu_\Ucal(A_{\phi(c_t)})}{2} \geq \frac{\epsilon}{2}.
    \end{equation}
    Next, we consider the online learning setup for which $x_1,\ldots,x_t$ are sampled independently from $\mu_1,\ldots,\mu_t$ respectively and consider the values $y_{t'}=f_{c_t}(x_{t'})$ for all $t'<t$. We now define the value
    \begin{equation}\label{eq:def_value_z}
        z_t:=\arg\min_{z\in\{0,1\}}\Pbb(\hat y_t^{(t)}=z, x_t\in \{c_t\}),
    \end{equation}
    where we denoted by $\hat y_t^{(t)}$ the prediction of the algorithm at time $t$ for this online learning instance.
    If $z=0$ we pose $b_t=c_t$ and fix $a_t\in(a_{t-1},c_t)$ such that
    \begin{equation}\label{eq:recursive_def}
        \Pbb_{x_{t'}\sim\mu_{t'}}\paren{\exists t'\in[t]: x_{t'}\in  (a_t,b_t)} \leq \epsilon/8.
    \end{equation}
    Note that this is indeed possible since as $a_t\to c_t$, the subset of $\Xcal$ corresponding to interval $(a_t,c_t)$ decreases to the empty set. Conversely, if $z_t=1$, we pose $a_t=c_t$ and fix any $b_t\in(c_t,b_{t-1})$ such that \cref{eq:recursive_def} holds. This ends the recursive construction of the sequences. By construction, note that $(a_t)_{t\geq 1}$ and $(b_t)_{t\geq 1}$ are convergent to the same real value which we denote $r^\star\in\Rbb$.

    We now check that for the sequence of instances $(\mu_t)_{t\geq 1}$ and function $f^\star:=f_{r^\star}\in\Fcal$, the learner induces linear regret. By construction, note that $f^\star$ has value $1-z_t$ on the set $\{c_t\}=A_{\phi(c_t)}$ for all $t\geq 1$. Therefore, for any $T\geq 1$,
    \begin{equation*}
        \Ebb\sqb{\sum_{t=1}^T\1[\hat y_t\neq f^\star(x_t)]} \geq \sum_{t=1}^T \Pbb(\hat y_t=z_t,x_t\in\{c_t\}).
    \end{equation*}
    Note that by construction we have $c_t,r^\star\in(a_{t-1},b_{t-1})$. In particular, $f^\star$ and $f_{c_t}$ yield the same values $y_{t'}$ for $t'<t$ whenever for all previous samples $t'<t$ one has $x_{t'}\notin(a_{t-1},b_{t-1})$. Hence, for any $t\geq 1$,
    \begin{align*}
        \Pbb(\hat y_t=z_t,x_t\in\{c_t\}) &\geq \Pbb(\hat y_t^{(t)}=z_t,x_t\in\{c_t\}) - \Pbb(\exists t'\in[t-1]: x_{t'}\in(a_{t-1},b_{t-1}))\\
        &\overset{(i)}{\geq} \frac{\Pbb(x_t\in\{c_t\})}{2} - \frac{\epsilon}{8} \overset{(ii)}{\geq} \frac{\epsilon}{8}.
    \end{align*}
    In $(i)$ we used \cref{eq:def_value_z,eq:recursive_def} and in $(ii)$ we used \cref{eq:lb_mass_A_c_t}. Plugging these estimates within the previous equation shows that for all $T\geq 1$,
    \begin{equation*}
        \Ebb\sqb{\sum_{t=1}^T\1[\hat y_t\neq f^\star(x_t)]} \geq \frac{\epsilon T}{8}.
    \end{equation*}
    This ends the proof that $\Ucal$ is not weakly learnable for $\Fcal$, reaching a contradiction. Hence $\mu_\Ucal$ is a continuous submeasure.
\end{proof}

\section{Further preliminaries}

\subsection{Additional definitions}

We give the definition of the Littlestone dimension below, which is known to characterize learnable function classes for fully unconstrained adversaries \cite{littlestone1988learning,ben2009agnostic}.

\begin{definition}[Littlestone Dimension] \label{def:littlestone}
    The \emph{Littlestone dimension} of a function class $\Fcal:\Xcal\to\{0,1\}$, denoted by $\mathrm{Ldim}(\Fcal)$, is the largest integer $d$ such that there exists a complete binary tree of depth $d$ with internal nodes labeled by elements of $\Xcal$ and edges labeled by $\{0,1\}$, such that for every root-to-leaf path $(x_1,b_1),\ldots,(x_d,b_d)$, there exists a function $f\in\Fcal$ satisfying $f(x_i)=b_i$ for all $i\in[d]$. If such trees exist for arbitrarily large $d$, then the Littlestone dimension is infinite.
\end{definition}

\subsection{Concentration inequlities}

Freedman's inequality \cite{freedman1975tail} gives tail probability bounds for martingales. The statement below is for instance taken from \cite[Theorem 1]{beygelzimer2011contextual} or \cite[Lemma 9]{agarwal2014taming}.

\begin{theorem}[Freedman's inequality]\label{thm:freedman_inequality}
        Let $(Z_t)_{t\in T}$ be a real-valued martingale difference sequence adapted to filtration $(\Fcal_t)_t$. If $|Z_t|\leq R$ almost surely, then for any $\eta\in(0,1/R)$, with probability at least $1-\delta$,
        \begin{equation*}
            \sum_{t=1}^T Z_t \leq \eta \sum_{t=1}^T \Ebb[Z_t^2\mid\Fcal_{t-1}] + \frac{\ln1/\delta}{\eta}.
        \end{equation*}
    \end{theorem}

\subsection{Properties of Generalized Smoothness}

In this section, we give further properties of generalized smoothness, starting with short proof of the equivalence between generalized smoothness and having uniform covers for all VC classes.

\vspace{3pt}

\begin{proof}{\textbf{of \cref{prop:uniform_covers_equal_gen_smoothness}}}
    We first check that a $(\mu_0,\rho)$-generalized smooth distribution class $\Ucal$ has uniform covers for any VC class $\Fcal$. Fix $\epsilon>0$. The distribution $\mu_0$ admits a finite cover $\Fcal_\epsilon$ at scale $\delta:=\rho^{-1}(\epsilon)>0$ \citep{haussler1995sphere}, that is, for any $f\in\Fcal$ there exists $\hat f\in\Fcal_\epsilon$ such that $\mu_0(\{x:f(x)\neq \hat f(x)\})\leq \delta$. By the generalized smoothness property, $\Fcal_\epsilon$ is therefore a uniform cover for $\Ucal$ for scale $\epsilon>0$.

    To complete the proof, it suffices to check that if $\Ucal$ has uniform covers for all generalized thresholds, then it is generalized smooth. Indeed, if it doesn't hold, \cref{lemma:continuous_submeasure_to_smooth} shows that $\mu_\Ucal$ is not a continuous submeasure. Following the standard argument in the beginning of the proof of \cref{lemma:threshold_to_continuous_submeasure}, this shows that there exists $\epsilon>0$ and a sequence of disjoint measurable sets $(A_i)_{i\geq 1}$ such that $\mu_\Ucal(A_i)\geq \epsilon$ for all $i\geq 1$. We then construct the generalized threshold class $\Fcal:=\{x\mapsto \1[x\in\bigcup_{j\geq i}A_j]:i\geq 1\}$ which by construction is an infinite $\epsilon$-packing for $\mu_\Ucal$. That is, for any $f\neq f'\in\Fcal$, there exists $\mu\in\Ucal$ such that $\mu(\{x:f(x)\neq f'(x)\}) \geq \epsilon$. In particular, $\Fcal$ does not admit finite uniform covers for scale $\epsilon/2$, ending the proof.
\end{proof}

Next, we show that the tolerance function $\rho$ for a $(\mu_0,\rho)$-generalized smooth distribution class $\Ucal$ is necessarily well-behaved if all distribution in $\Ucal$ are non-atomic.

\begin{proposition}\label{prop:non_atomic_concavity}
    Let $\Ucal$ be a $(\mu_0,\rho)$-generalized smooth distribution class where $\rho:[0,1]\to\Rbb_+$ is non-decreasing and $\lim_{\epsilon\to0}\rho(\epsilon)=0$, such that all $\mu\in\Ucal$ are non-atomic (i.e., $\mu(A)=0$ for all atoms $A$). Then, the function $\tilde \rho$ defined via
    \begin{equation*}
        \tilde \rho(\epsilon) = \sup\{\mu(A): \mu\in\Ucal,A\in\Bcal,\mu_0(A)\leq \epsilon\},\quad \epsilon\geq 0,
    \end{equation*}
    is well-behaved and $\Ucal$ is $(\mu_0,\tilde \rho)$-generalized smooth. Also, $\tilde \rho(\epsilon)\leq \rho(\epsilon)$ for all $\epsilon\geq 0$.
\end{proposition}

\begin{proof}
    First, since $\Ucal$ is $(\mu_0,\rho)$-generalized smooth, for any $\epsilon\geq 0$, $\mu\in\Ucal$ and $A\in\Bcal$ with $\mu_0(A)\leq \epsilon$ we have $\mu(A)\leq \rho(\epsilon)$. Hence, by construction of $\tilde\rho$, we have $\tilde\rho(\epsilon)\leq \rho(\epsilon)$. Also, by construction we directly have $\tilde\rho(0)=0$. 
    
    It remains to check that $\tilde\rho$ is well-behaved. To do so, we first show that for any non-atomic distribution $\mu\ll \mu_0$ on $\Xcal$, $A\in\Bcal$ and $\delta\in[0,\mu_0(A)]$, there exists $B\subseteq A$ such that $\mu_0(B)\leq\delta$ and $\mu(B) \geq \mu(A)\frac{\delta}{\mu_0(A)}$. To do so, for any $z\geq 0$ we define the set $B_z:=A\cap \{x: \frac{d\mu}{d\mu_0}(x) \geq z\}$. Since $\mu\ll \mu_0$ we have $B_z\downarrow 0$ as $z\to\infty$, while $B_0=A$. Hence, there exists $z_\delta\geq 0$ such that $\mu_0(B_{z_\delta}) \geq \delta$ but $\mu_0(A\cap \{x: \frac{d\mu}{d\mu_0}(x) > z_\delta\})\leq \delta$. As a result, there exists a set $B$ such that
    \begin{equation*}
        A\cap \left\{x: \frac{d\mu}{d\mu_0}(x) > z_\delta\right\} \subseteq B \subseteq A\cap \left\{x: \frac{d\mu}{d\mu_0}(x) \geq z_\delta\right\} = B_{z_\delta}
    \end{equation*}
    and an atom $C\in\Bcal$ such that $B\cup C\subseteq B_{z_\delta}$, $\mu_0(B)\leq\delta$ but $\mu_0(B\cup C)\geq \delta$. Now we compute
    \begin{align*}
        \mu(A) = \mu(B\cup C) +\mu(A\setminus (B\cup C)) &\overset{(i)}{\leq} \mu(B\cup C) +z_\delta \mu_0(A\setminus (B\cup C))\\
        &\overset{(ii)}{\leq} \mu(B\cup C) +\frac{\mu(B\cup C)}{\mu_0(B\cup C)} \mu_0(A\setminus (B\cup C)) \overset{(iii)}{=} \mu(B) \frac{\mu_0(A)}{\mu_0(B\cup C)}.
    \end{align*} 
    In $(i)$ we used the fact that $A\setminus (B\cup C)\subseteq\{x:\frac{d\mu}{d\mu_0}(x) \leq z_2\}$ and in $(ii)$ we used $B\cup C\subseteq B_{z_\delta}$. In $(iii)$ we used $\mu(B)=\mu(B\cup C)$ since $\mu$ is non-atomic. This exactly gives the desired bound $\mu(B)\geq \mu(A)\frac{\delta}{\mu_0(A)}$ while $\mu_0(B)\leq \delta$.

    With this property at hand, we can check that $\tilde\rho(\epsilon)/\epsilon$ is non-increasing in $\epsilon$. First, note that all $\mu\in\Ucal$ satisfy $\mu\ll \mu_0$ otherwise we would have $\lim_{\epsilon\to 0}\rho(\epsilon)>0$. Also, by assumption, all distributions $\mu\in\Ucal$ are non-atomic. Now fix any $0<\epsilon_1\leq\epsilon_2$. For any $\eta>0$ there is $\mu\in\Ucal$, $A\in\Bcal$ with $\mu_0(A)\leq \epsilon_2$ and $\mu(A)\geq \tilde\rho(\epsilon_2)-\eta$. From the previous paragraph, there is $B\subseteq A$ with $\mu_0(B)\leq \epsilon_1$ and $\mu(B)\geq (\tilde\rho(\epsilon_2)-\eta)\frac{\epsilon_1}{\epsilon_2}$. This holds for any $\eta>0$ and hence $\tilde\rho(\epsilon_1)\geq \tilde\rho(\epsilon_2)\frac{\epsilon_1}{\epsilon_2}$.
\end{proof}

\section{Proofs from \cref{sec:universal_algs}}

We start by proving the reduction coupling \cref{lemma:coupling} from generalized smooth adversaries to (standard) smooth adversaries.

\vspace{3pt}

\begin{proof}{\textbf{of \cref{lemma:coupling}}}
    Denote by $\mu_0$ the base measure for the generalized smoothed class $\Ucal$. We also denote by $\tilde\mu_0:=\frac{1}{2}\mu_0 + \frac{1}{2}\delta_{x_\emptyset}$ the mixture of the base measure and a Dirac at the dummy variable $x_\emptyset$. For convenience, we denote $\eta:=\rho(\epsilon)/\epsilon$.
    For any $t\in[T]$, denote by $\mu_t$ the distribution of $x_t$ conditional on the past history $x_{<t}$. We then construct the sample $\tilde x_t$ via
    \begin{equation*}
        \tilde x_t:=\begin{cases}
            x_t &\text{if }\frac{d\mu_t}{d\mu_0}(x_t) \leq \eta\\
            x_\emptyset &\text{otherwise}.
        \end{cases}
    \end{equation*}

    Denote by $\tilde\mu_t$ the distribution of $\tilde x_t$ conditional to $x_{<t}$.
    We can immediately check that for any $x\in\Xcal$,
    \begin{equation*}
        \frac{d\tilde\mu_t}{d\tilde \mu_0}(x) \leq 2\frac{d\tilde\mu_t}{d\mu_0}(x) = 2\frac{d\mu_t}{d\mu_0}(x) \1\sqb{\frac{d\mu_t}{d\mu_0}(x) \leq \eta} \leq 2\eta.
    \end{equation*}
    On the other hand, we clearly have $\frac{d\tilde\mu_t}{d\tilde \mu_0}(x_\emptyset)\leq 2$ since $\tilde \mu_0$ places $1/2$ mass onto $x_\emptyset$. This ends the proof that $\tilde x_1,\ldots,\tilde x_T$ is $\frac{1}{2\eta}$-smooth with respect to $\tilde\mu_0$. Next, for any $t\in[T]$,
    \begin{align*}
        \Pbb[\tilde x_t=x_\emptyset\mid \tilde x_{<t},x_{<t}] = \Pbb\sqb{\frac{d\mu_t}{d\mu_0}(x_t) > \eta \mid x_{<t}} = \mu_t\paren{\left\{\frac{d\mu_t}{d\mu_0}(x) > \eta \right\}} .
    \end{align*}
    For convenience, denote $A_t:=\{\frac{d\mu_t}{d\mu_0}(x) > \eta \}$. Then, by construction, $\eta \mu_0(A_t) < \mu_t(A_t) \leq \rho(\mu_0(A_t))$,
    where in the last inequality we used $\mu_t\in\Ucal$ and the generalized smoothed property. 
    Recalling that $\eta = \frac{\rho(\epsilon)}{\epsilon}$ and $\alpha\mapsto\frac{\rho(\alpha)}{\alpha}$ is non-increasing, we obtain $\mu_0(A_t)\leq \epsilon$ and in turn, $\mu_t(A_t)\leq \rho(\epsilon)$. This ends the proof.
\end{proof}

Next, we prove \cref{thm:universal_realizable} which gives regret bounds for ERM under any VC class and generalized smooth distribution class $\Ucal$.

\vspace{3pt}

\begin{proof}{\textbf{of \cref{thm:universal_realizable}}}
    We will use \cref{lemma:coupling} to reduce the problem to smooth adversaries, for which \cite{block2024performance} gives regret lower bounds. To do so, we fix $\epsilon>0$ and letting $\sigma:=\frac{\epsilon}{2\rho(\epsilon)}$, we use \cref{lemma:coupling} to construct the $\sigma$-smooth sequence $\tilde x_1,\ldots,\tilde x_T$ as guaranteed by this result. We then extend the function class $\Fcal$ to $\Xcal\cup\{x_\emptyset\}$ by posing $f(x_\emptyset)=0$ for all $f\in\Fcal$. Throughout, we use tildes to distinguish between quantities computed for the smooth sequence $\tilde x$ or the original sequence $x$.
    
    Denote by $f_t$ the prediction function chosen by the ERM algorithm at time $t$ (which was trained on the available data $(x_{t'},f^\star(x_{t'}))_{t'<t}$). 
    The main observation is that $f_t$ would still be an ERM prediction function on the training data $(\tilde x_{t'},f^\star(\tilde x_{t'}))_{t'<t}$ by construction of $\tilde x_{<t}$ and since all functions coincide on $x_\emptyset$. For convenience, denote $\tilde y_t:=f_t(\tilde x_t)$. Since the sequence $\tilde x$ is $\sigma$-smooth, the regret bound from \cite[Theorem 8]{block2024performance} applies:
    \begin{equation*}
        \Ebb\sqb{\sum_{t=1}^T \1[\tilde y_t\neq f^\star(\tilde x_t)]} \leq C \sqrt{\frac{d}{\sigma}T  \ln^5 T} = C \sqrt{\frac{2\rho(\epsilon)}{\epsilon} dT  \ln^5 T} 
    \end{equation*}
    for some constant $C>0$.
    As a slight note, the \cite[Theorem 8]{block2024performance} gives a dependency of the regret in $1/\sigma$, but this can be improved to $1/\sqrt \sigma$ with the same proof as noted in \cite{blanchard2025agnostic}. As a result, we obtain
    \begin{equation*}
        \Ebb\sqb{\sum_{t=1}^T \ell_t(\hat y_t)} \leq \Ebb\sqb{\sum_{t=1}^T \1[\tilde y_t\neq f^\star(\tilde x_t)] + \1[\tilde x_t=x_\emptyset]} \leq  C \sqrt{\frac{2\rho(\epsilon)}{\epsilon} dT  \ln^5 T}  + \rho(\epsilon)T,
    \end{equation*}
    where in the last inequality we used the last property from \cref{lemma:coupling}. Taking the infimum over $\epsilon>0$ ends the proof.
\end{proof}

Finally, we prove the regret bound on \textsc{R-Cover} from \cref{thm:regret_R_cover} for agnostic distributionally constrained adversaries.

\vspace{3pt}

\begin{proof}{\textbf{of \cref{thm:regret_R_cover}}}
    Again, we use \cref{lemma:coupling} to reduce the problem to smooth adversaries, for which \cite{blanchard2025agnostic} gives regret lower bounds. We use the same notations for the $\sigma:=\frac{\epsilon}{2\rho(\epsilon)}$-smooth sequence $\tilde x_1,\ldots,\tilde x_T$ as in the proof of \cref{thm:universal_realizable}.

    Importantly, the only part of the original proof which uses smoothness is in \cite[Lemma 11]{blanchard2025agnostic} which bounds some discrepancy terms denoted $\Gamma_k^{(p,r)}$ (for both classification and regression), and when going from oblivious regret bounds to adaptive regret bounds (only for classification). In particular, their regret decomposition still holds even for the regression setup. 
    We rewrite these bounds here for completeness. For context, the \textsc{R-Cover} algorithm proceeds by epochs on several layers $p\in\{0,\ldots,P=\floor{\log_2 T}\}$: these layer-$p$ epochs are denoted $E_1^{(p)},\ldots,E_{N_p}^{(p)}$ and form a roughly-balanced partition of $[T]$ and $N_p=2^{P-p}$. Their end points are denoted $E_k^{(p)}:=(T_{k-1}^{(p)},T_k^{(p)}]$ for $k\in[N_p]$. As a first step of the original proof (see \cite[Equation (14)]{blanchard2025agnostic}, the regret of the learning-with-experts algorithm at layer $p$ during epoch $E_k^{(p)}$ for $k\in[N_p]$ are bounded using the following quantity (with simplified notation)
    \begin{equation*}
        \Delta_k^{(p)}:= \sum_{t\in E_p^{(k)}} \Ebb_{r\sim\hat p_t} \sqb{(\ell_t(\hat y_t ) - \ell_t(\Acal_{r}(t) ))^2}.
    \end{equation*}
    Here, $\hat p_t$ is the distribution among experts used by the algorithm at time $t$, $\hat y_t$ denotes the prediction of the expert, and $\Acal_{l,r}(t)$ denotes the prediction of expert $r$ at time $t$. The main point is that we can bound
    \begin{equation}\label{eq:bound_Delta}
        \Delta_k^{(p)} \leq \bar\Delta_k^{(p)} + \sum_{t\in E_p^{(k)}} \1[\tilde x_t=x_\emptyset] \quad \text{where} \quad \bar\Delta_k^{(p)} := \sum_{t\in E_p^{(k)}} \Ebb_{r\sim\hat p_t} \sqb{(\ell_t(\hat y_t ) - \ell_t(\Acal_{r}(t) ))^2}\1[\tilde x_t=x_t]. 
    \end{equation}
    As a note, $\bar\Delta_k^{(p)}$ is distinct from the quantity $\tilde\Delta_k^{(p)}$ which would correspond to using the smoothed sequence $\tilde x_t$ throughout the game instead of $x_t$, since \textsc{R-Cover} constructs covers (which affect $\hat p_t$) using the available data $ x_t$.
    Similarly, we can bound the following quantity which is also used within the original proof:
    \begin{equation}\label{eq:bound_Lambda}
        \Lambda_k^{(p)}:= \sum_{t\in E_k^{(p)}} \ell_t\paren{f_{k,S}^{(p)}(x_t) } - \ell_t(f^\star(x_t) ) \leq \tilde \Lambda_k^{(p)} + \sum_{t\in E_p^{(k)}} \1[\tilde x_t=x_\emptyset],
    \end{equation}
    where we recall that $\tilde\Lambda_k^{(p)}$ denotes the quantity to $\Lambda_k^{(p)}$ but for the smoothed sequence $\tilde x_t$.

    We further the bounds as in the original proof, using \cite[Lemma 10]{blanchard2025agnostic}. Formally, we first define    
    \begin{equation*}
        \Pcal_k^{(p)}:=\{ f\in\Fcal: \max_{t\in[T_{k-1}^{(p)}]} \abs{f(x_t)-f^\star(x_t)} \leq 2\alpha\},
    \end{equation*}
    which corresponds to the set of functions approximately consistent with $f^\star$ on instances from previous layer-$p$ epochs. Then, for any $r\geq 1$, we define
    \begin{equation*}
         \bar\Gamma_k^{(p,r)} := \sum_{t\in E_k^{(p)}} \bar\gamma^{(p,r)}(t) \quad\text{where}\quad \bar\gamma^{(p,r)}(t) := \sup_{f\in\Pcal_k^{(p)}}\Ebb\sqb{ |f(x_t)-f^\star(x_t)|^r \1[\tilde x_t=x_t]\mid\Hcal_{t-1}}, \quad t\in E_k^{(p)}
    \end{equation*}
    which intuitively quantifies the $\ell_r$ discrepancy between the queries on epoch $E_k^{(p)}$ and queries prior to this epoch. Since within $\bar \Delta_k^{(p)}$, we only focus on times when $\tilde x_t$ and $x_t$ coincide, the same proof as for \cite[Lemma 10]{blanchard2025agnostic} shows that with probability at least $1-\delta$, we have $\bar \Delta_k^{(p)} \leq 5\bar\Gamma_k^{(p,2)} + 16\ln\frac{T}{\delta}$ for all $p\in[P],k\in[N_p]$.
    A main observation is that since all functions agree on $x_\emptyset$, observing the value of $f^\star$ on this instance does not bring any new information. Formally, for all layers $p$ and epochs $k$, $\Pcal_k^{(p)} \subseteq\tilde\Pcal_k^{(p)}$. Hence, for any $t\in R_k^{(p)}$, one has
    \begin{equation*}
        \bar\gamma^{(p,r)}(t) \leq \tilde\gamma^{(p,r)}(t):= \sup_{f\in\tilde \Pcal_k^{(p)}}\Ebb\sqb{ |f(\tilde x_t)-f^\star(\tilde x_t)|^r\mid\Hcal_{t-1}}.
    \end{equation*}
    In particular, we obtained $\bar\Gamma_k^{(p,r)}\leq \tilde\Gamma_k^{(p,r)}:=\sum_{t\in E_k^{(p)}} \tilde\gamma^{(p,r)}(t) $ for all $r\geq 1$, which corresponds to the quantity used within the original proof if we were working with the smoothed data $\tilde x_t$.
    
    Altogether, the regret decomposition bound from \cite[Equation (14)]{blanchard2025agnostic} shows that for any $p_0\in\{0,\ldots,P\}$, with probability at least $1-4\delta$,
    \begin{align*}
        &\sum_{t=1}^T \ell_t(\hat y_t ) - \ell_t(f^\star(x_t) ) \\
    &\leq 12\sum_{p=\max(p_0,1)}^{P} \sum_{k\in[N_p]} \sqrt{2\max\paren{\Delta_k^{(p)},2} \ln\paren{\Ncal(\Fcal;\alpha,T)+1 }} + 8N_{p_0}\ln\frac{2T}{\delta} +\sum_{k\in[N_{p_0}]} \Lambda_k^{(p)}\\
     &\overset{(i)}{\leq}  12\sum_{p=\max(p_0,1)}^{P} \sum_{k\in[N_p]} \sqrt{2\max\paren{\bar \Delta_k^{(p)},2} \ln\paren{\Ncal(\Fcal;\alpha,T)+1 }} + 8N_{p_0}\ln\frac{2T}{\delta} +\sum_{k\in[N_{p_0}]} \tilde\Lambda_k^{(p)}\\
     &  + (12P\sqrt{2 \ln\paren{\Ncal(\Fcal;\alpha,T)+1 }} + 1)\sum_{t\in[T]}\1[\tilde x_t=x_\emptyset]\\
     &\leq 12\sum_{p=\max(p_0,1)}^{P} \sum_{k\in[N_p]} \sqrt{2\paren{5\tilde \Gamma_k^{(p,2)}+ 16\ln\frac{T}{\delta}} \ln\paren{\Ncal(\Fcal;\alpha,T)+1 }} + 8N_{p_0}\ln\frac{2T}{\delta} \\
        &+\sum_{k\in[N_{p_0}]}\paren{2\tilde \Gamma_k^{(p_0,1)} + 3\ln\frac{T}{\delta}}+ (12P\sqrt{2 \ln\paren{\Ncal(\Fcal;\alpha,T)+1 }} + 1)\sum_{t\in[T]}\1[\tilde x_t=x_\emptyset].
    \end{align*}
    where $\Ncal(\Fcal;\alpha,T)$ is the $\alpha$-covering number of $\Fcal$ and $\alpha$ is the parameter used within the algorithm to construct coverings. In $(i)$ we used \cref{eq:bound_Delta,eq:bound_Lambda} together with the inequalities $\sqrt{a+b}\leq \sqrt a+\sqrt b$ for $a,b\geq 0$ and $\sqrt a\leq a$ for any non-negative integer.
    Importantly, except for the term proportional to $\sum_{t\in[T]}\1[\tilde x_t=x_\emptyset]$ the rest of the last term exactly corresponds to the regret bound from the original proof for the smoothed data $\tilde x_t$. In particular, the rest of the proof directly applies. Also, in the binary classification setting, we can take the parameter $\alpha=0$ and by Sauer-Shelah lemma \citep{sauer1972density,shelah1972combinatorial} we have $\ln\Ncal(\Fcal;\alpha,T)\lesssim d\ln T$. In particular, in the binary classification setting, the rest of the proof of \cite[Proposition 13]{blanchard2025agnostic} implies that for any $f^\star\in\Fcal$, with probability at least $1-\delta$,
    \begin{align*}
         \sum_{t=1}^T \ell_t(\hat y_t ) - \sum_{t=1}^T \ell_t(f^\star(x_t) ) \leq C \sqrt{\frac{(d\ln^2 T + d\ln\ln\frac{1}{\delta} + \ln\frac{1}{\delta})\ln^3 T}{\sigma}\cdot T} + C\sqrt{d\ln^3 T} \sum_{t=1}^T\1[\tilde x_t=x_\emptyset],
    \end{align*}
    for some universal constant $C>0$.
    Further, from the last claim of \cref{lemma:coupling}, we can apply Freedman's inequality \cref{thm:freedman_inequality} to the martingale increments $\1[\tilde x_t = x_\emptyset] - \Pbb[\tilde x_t=x_\emptyset\mid \tilde x_{<t},x_{<t}]$ to show that with probability at least $1-\delta$,
    \begin{equation}\label{eq:coupling_failure_bound}
        \sum_{t=1}^T\1[\tilde x_t=x_\emptyset] \leq 2\rho(\epsilon)T + 2\ln \frac{1}{\delta}.
    \end{equation}

    The last step of the proof is to generalize this result for oblivious benchmark $f^\star$ to adaptive benchmarks (where $f^\star$ may be data-dependent). The same arguments from \cite[Section 5.2]{blanchard2025agnostic} apply here with minor adjustments. As in the original proof, we construct an $\alpha$-cover $\Hcal$ of $\Fcal$ for the base measure $\tilde\mu_0$ with $\ln|\Hcal|\leq 2d\ln(e^2/\alpha)$ for $\alpha=\frac{\sigma}{T} \ln\frac{T}{\delta}$ and bound the adaptive regret using the regret of the algorithm compared to functions within $\Hcal$ (see \cite[Eq (30)]{blanchard2025agnostic}):
    \begin{equation*}
        \sum_{t=1}^T \ell_t(\hat y_t ) - \inf_{f\in\Fcal} \sum_{t=1}^T \ell_t(f(x_t) ) \leq \sum_{t=1}^T \ell_t(\hat y_t ) - \inf_{f\in\Hcal} \sum_{t=1}^T \ell_t(f(x_t) ) +
        \sup_{g\in\Gcal}\sum_{t=1}^T g(x_t),
    \end{equation*}
    where $\Gcal:=\{\1[f\neq h_f],f\in\Fcal\}$. The regret compared to $\Hcal$ is bounded using the previous fixed-benchmark-high-probability regret bound and the union bound, as in the original proof. We then note that
    \begin{equation*}
        \sup_{g\in\Gcal}\sum_{t=1}^T g(x_t) \leq \sum_{t=1}^T\1[\tilde x_t= x_\emptyset] + \sup_{g\in\Gcal}\sum_{t=1}^T g(\tilde x_t).
    \end{equation*}
    Hence, from there we can use the same arguments as in the original proof to bound the last term. Putting everything together as in the original proof, together with \cref{eq:coupling_failure_bound}, we obtained with probability at least $1-3\delta$,
    \begin{equation*}
        \sum_{t=1}^T \ell_t(\hat y_t ) - \inf_{f\in\Fcal} \sum_{t=1}^T \ell_t(f(x_t) )\lesssim  \sqrt {\frac{(d\ln^2 T + d\ln\ln\frac{1}{\delta} + \ln\frac{1}{\delta}) \ln^3 T}{\sigma} \cdot T} 
        +  \rho(\epsilon)T\sqrt{d\ln^3 T}.
    \end{equation*}
    This ends the proof.
\end{proof}

\section{Proofs from \cref{sec:dist_dep_bounds}}

We first prove \cref{lemma:continuous_submeasure_to_smooth_quantitative} which bounds the modulus of continuity of the envelope functional $\mu_\Ucal$ whenever $\Ucal$ has finite fragmentation numbers. This gives a quantitative version of \cref{lemma:continuous_submeasure_to_smooth}.

\vspace{3pt}

\begin{proof}{\textbf{of \cref{lemma:continuous_submeasure_to_smooth_quantitative}}}
    The proof follows the same iterative construction as \cref{lemma:continuous_submeasure_to_smooth}, but with parameters tuned to track the dependence on $N_{\Ucal}(\epsilon)$. Fix $\Ucal$ and $\epsilon\in(0,1]$ satisfying the assumptions. We iteratively construct a sequence of distributions $(\mu_i)_{i\geq 1}$ together with measurable sets $(A_i)_{i\geq 1}$ as follows. Suppose we have constructed $\mu_j,A_j$ for all $j<i$ for some $i\geq 1$. If there exists a measurable set $A$ such that
    \begin{equation*}
        \sum_{j<i}\mu_j(A) \leq \frac{\epsilon}{N_\Ucal(\epsilon)} \quad \text{and}\quad \exists \mu\in\Ucal \text{ with } \mu(A) \geq 2\epsilon,
    \end{equation*}
    we fix $A_i$ satisfying this property and $\mu_i\in\Ucal$ with $\mu_i(A_i)\geq 2\epsilon$. Otherwise, we end the recursive construction. Next, if $A_i$ is defined, we introduce $B_i:=A_i\setminus \bigcup_{j>i}A_j$. By construction, all sets $B_i$ are disjoint and if $i\leq N_\Ucal(\epsilon)+1$,
    \begin{equation*}
        \mu_i(B_i)\geq \mu_i(A_i) - \sum_{j>i} \mu_i(A_j) \geq 2\epsilon - (i-1) \frac{\epsilon}{N_\Ucal(\epsilon)}  \geq \epsilon.
    \end{equation*}
    By definition of $N_\Ucal(\epsilon)$, this implies that the construction must have ended at some index $i_{\max}\leq N_\Ucal(\epsilon)$. We then pose $\mu_\epsilon:= \frac{1}{i_{\max}}\sum_{j=1}^{i_{\max}} \mu_j$. Then, for any measurable set $A\subseteq\Xcal$ with $\mu_\epsilon(A)\leq \frac{\epsilon}{N_\Ucal(\epsilon)^2}$ we have
    \begin{equation*}
        \sum_{j=1}^{i_{\max}} \mu_j(A) \leq i_{\max} \frac{\epsilon}{N_\Ucal(\epsilon)^2} \leq \frac{\epsilon}{N_\Ucal(\epsilon)}.
    \end{equation*}
    Since the construction ended at $i_{\max}$, this precisely implies that $\mu_\Ucal(A)\leq 2\epsilon$, ending the proof.
\end{proof}

We next prove the regret lower bound from \cref{thm:distrib_dependent_lower_bound} based on fragmentation numbers.

\vspace{3pt}

\begin{proof}{\textbf{of \cref{thm:distrib_dependent_lower_bound}}}
    The proof uses classical arguments from \cite{ben2009agnostic} which gives a lower bound $\Omega(\sqrt{\text{LD}(\Fcal) T}\wedge T)$ on the regret for adversarial online learning for a function class $\Fcal$. 
    We recall the definition of the Littlestone dimension $\text{LD}(\Fcal)$ in \cref{def:littlestone} in the appendix for completeness.
    Fix $\epsilon>0$ and for convenience, denote $N=N_\Ucal(\epsilon)$. Let $A_1,\ldots,A_N$ be disjoint measurable sets such that for all $i\in[N]$, $\mu_\Ucal(A_i)\geq \epsilon$. Let $I_j:=\{i\in[N]: i\in[(j-1)N/d,jN/d]\}$ for $j\in[d]$. For convenience, we enumerate $I_j:=\{i_l^{(j)}, l\in[N_j]\}$. For each $l\in\{0,\ldots,N_1\}\times\ldots\times \{0,\ldots,N_d\}$, we define the function $f_l:\Xcal\to\{0,1\}$ via
    \begin{equation*}
        f_l(x):=\begin{cases}
            1 &\text{if }x\in \bigcup_{l>l_j} A_{i_l^{(j)}},\; j\in[d]\\
            0 &\text{otherwise}.
        \end{cases}
    \end{equation*}
    We then consider the function class $\Fcal$ which contains all such functions. For convenience, we also consider a simplified version of this function class $\Fcal':I_1\sqcup\ldots\sqcup I_d\to\{0,1\}$ where we identify all sets $A_i$ to a single point $i$: it is the product of thresholds on each set $I_j$ for $j\in[d]$.
    The Littlestone dimension of both function classes is the sum of the Littlestone dimension of each underlying threshold function class:
    \begin{equation}\label{eq:bound_Litt_dim}
        \text{LD}(\Fcal)=\text{LD}(\Fcal') \geq \sum_{j\in[d]} \ceil{\log_2 (|I_j|+1)} \gtrsim (d\wedge N)\log_2(\ceil{N/d}+1).
    \end{equation}
    Fix a learning algorithm $alg$. We then emulate the proof of the oblivious regret lower bound for online learning on this function class for a horizon $T':=\ceil{\epsilon T/2}$ as follows. For each new iteration $t$, let $i_t\in I_1\sqcup\ldots\sqcup I_d$ be the next element which would be chosen by the online learning lower bound adversary. At time $t$, we then define $\mu_t$ to be a distribution such that $\mu_t(A_{i_t})\geq \epsilon$, which exists by construction. If the sample $x_t\sim\mu_t$ satisfies $x_t\in A_{i_t}$, we use the same labeling $y_t$ as the online learning lower bound adversary. Otherwise, we randomly label $y_t\sim\Bcal(1/2)$ and ignore this iteration for the online learning adversary. Conditionally on the event $\Ecal$ that we were able to emulate $T'$ iterations of the online learning adversary, the regret lower bound for online learning \citep{ben2009agnostic} implies the desired lower bound
    \begin{equation*}
        \text{Reg}_T(\text{alg};\Fcal,\Ucal \mid \Ecal) \gtrsim \sqrt{\textrm{LD}(\Fcal)T'}\wedge T',
    \end{equation*}
    since the ignored iterations can only increase the regret in expectation. Here, $\text{Reg}_T(\text{alg};\Fcal,\Ucal \mid \Ecal)$ denotes the conditional regret on $\Ecal$. By construction, each iteration $t\in[T]$ emulates a step of the online learning adversary with probability $\epsilon$. Therefore, by Azuma-Hoeffding's inequality,
    \begin{equation*}
        \Pbb[\Ecal^c] = \Pbb\sqb{\sum_{t\in[T]}\1[x_t\in A_{i_t}] <\frac{\epsilon T}{2}} \overset{(i)}{\leq} \Pbb_{Y\sim\Bcal(T,\epsilon)}\sqb{Y<\frac{\epsilon T}{2}} \overset{(ii)}{\leq} e^{-\epsilon T/8}.
    \end{equation*}
    In $(i)$ we stochastically dominated the quantity of interest by a binomial of parameters $T$ and $\epsilon$, and in $(ii)$ we used Chernoff's inequality. Therefore, the final regret lower bound becomes
    \begin{align*}
        \text{Reg}_T(\text{alg};\Fcal,\Ucal) &\geq \text{Reg}_T(\text{alg};\Fcal,\Ucal \mid \Ecal) (1-\Pbb[\Ecal^c]) \\
        &\gtrsim \sqrt{\text{LD}(\Fcal)\epsilon T}\wedge (\epsilon T)\\
        &\gtrsim  \sqrt{\epsilon (d\wedge N_\Ucal(\epsilon)) T \ln \paren{\frac{N_\Ucal(\epsilon)}{d}+e}}\wedge (\epsilon T),
    \end{align*}
    where in the last inequality we used \cref{eq:bound_Litt_dim}. This implies the desired result when $3d\leq N_\Ucal(\epsilon)$.
\end{proof}

Finally, we prove the regret upper bound from \cref{thm:distrib_dependent_upper_bound} for $\Ucal$-dependent algorithms. This uses a quantitative variant of the coupling \cref{lemma:coupling} whose proof is deferred to \cref{sec:quantitative_coupling_lemma}.

\vspace{3pt}

\begin{proof}{\textbf{of \cref{thm:distrib_dependent_upper_bound}}}
    Fix a distribution class $\Ucal$ and a function class $\Fcal$ of VC dimension $d$.
    Fix $\epsilon>0$. From \cref{lemma:continuous_submeasure_to_smooth_quantitative}, we can construct a distribution $\mu_\epsilon$ such that for any measurable set $A\subseteq\Xcal$ with $\mu_\epsilon(A)\leq \frac{\epsilon}{N_\Ucal(\epsilon)^2}$, one has $\mu_\Ucal(A)\leq 2\epsilon$.

    Using \cref{lemma:coupling_v2}, we then couple the original adversary with a sequence $\tilde x_1,\ldots,\tilde x_T$ generated by an $(\epsilon/(4N_\Ucal(\epsilon)^2))$-smooth adversary (with respect to a base measure $\tilde\mu_\epsilon$) such that for all $t\in[T]$,
    \[\Pbb[\tilde x_t\neq x_t\mid \tilde x_{<t},x_{<t}]\leq 2\epsilon.\]
    This reduces the analysis to the smoothed model, plus an additive term accounting for coupling failures.

    Specifically, consider the algorithm $alg_\alpha$ that runs Hedge over an $\alpha$-cover of $\Fcal$ with respect to the base measure $\tilde\mu_\epsilon$; denote such a cover by $\Fcal(\alpha)$. Here, $\alpha>0$ is a tuning parameter. Following \cite[Theorem 3.1]{haghtalab2020smoothed}, we can bound the expected regret of $alg_\alpha$ by
    \begin{align*}
        &\Ocal\paren{\sqrt{T\ln|\Hcal(\alpha)|}} + \Ebb\sqb{\max_{f\in\Fcal}\min_{g\in\Fcal(\alpha)}\sum_{t=1}^T\1[f(x_t)\neq g(x_t)]}\\
        &\leq \Ocal\paren{\sqrt{T\ln|\Hcal(\alpha)|}}  + \Ebb\sqb{\max_{f\in\Fcal}\min_{g\in\Fcal(\alpha)}\sum_{t=1}^T\1[f(\tilde x_t)\neq g(\tilde x_t)]} + \Ebb\sqb{\sum_{t=1}^T\1[\tilde x_t\neq x_t]}.
    \end{align*}
    By construction, the last term is bounded by $2\epsilon T$. The first two terms exactly correspond to the regret analysis for the smoothed data $\tilde x_{\leq T}$. In turn, the arguments from \cite[Theorem 3.1]{haghtalab2020smoothed} directly imply that with an appropriate value of $\alpha$,
    \begin{equation*}
        \text{Reg}_T(alg_\alpha;\Fcal,\Ucal) \lesssim \sqrt{Td \ln \paren{\frac{4TN_\Ucal(\epsilon)^2}{d\epsilon}}} + d \ln \paren{\frac{4TN_\Ucal(\epsilon)^2}{d\epsilon}} + \epsilon T.
    \end{equation*}
    We then minimize this value over $\epsilon\in(0,1]$. First, note that without loss of generality we may focus on $\epsilon\geq 1/T$: smaller values cannot improve the right-hand side by more than a constant factor, since the first term is already $\Omega(1)$ in that regime.

    Next, we can ignore the second term on the right-hand side: if it dominates the first term, then necessarily $d \ln \paren{\frac{4TN_\Ucal(\epsilon)^2}{d\epsilon}}\gtrsim T$, in which case the bound is already trivial (it exceeds $T$).
    Tuning the parameter $\epsilon\in(1/T,1]$ then gives the desired result.
\end{proof}

\section{A quantitative coupling lemma}
\label{sec:quantitative_coupling_lemma}

The next lemma is a variant of \cref{lemma:coupling} that is convenient when one only has a guarantee of the form
$\mu_0(A)\leq \epsilon \Rightarrow \mu_{\Ucal}(A)\leq \eta$ for some base measure $\mu_0$.

\begin{lemma}\label{lemma:coupling_v2}
    Let $\Ucal$ be a generalized smoothed distribution class on $\Xcal$, let $\epsilon,\eta\in(0,1]$, and let $\mu_0$ be a probability measure such that for any measurable set $A\subseteq\Xcal$, if $\mu_0(A)\leq \epsilon$ then $\mu_\Ucal(A)\leq \eta$. Consider any $\Ucal$-constrained adaptive adversary generating samples $x_1,\ldots,x_T$. Fix a dummy symbol $x_\emptyset\notin\Xcal$.

    Then there exists a coupling with a sequence $\tilde x_1,\ldots,\tilde x_T\in\Xcal\cup\{x_\emptyset\}$ which is generated by an $(\epsilon/4)$-smooth adversary on $\Xcal\cup\{x_\emptyset\}$ such that for all $t\in[T]$,
    \begin{enumerate}
        \item $\tilde x_t\in\{x_t,x_\emptyset\}$, and
        \item $\Pbb[\tilde x_t=x_\emptyset\mid \tilde x_{<t},x_{<t}]\leq \eta$.
    \end{enumerate}
\end{lemma}

\begin{proof}
    The proof is similar to that of \cref{lemma:coupling}.
    Denote by $\mu_t$ the distribution of $x_t$ conditional on the past history $x_{<t}$. Define the mixture base measure
    \[\tilde\mu_0 := \tfrac{1}{2}\mu_0 + \tfrac{1}{2}\delta_{x_\emptyset}.
    \]
    For convenience, set $\sigma:=\eta/\epsilon$.

    For each $t\in[T]$, construct $\tilde x_t$ from $x_t$ as
    \[
                	\tilde x_t:=\begin{cases}
            x_t &\text{if }\frac{d\mu_t}{d\mu_0}(x_t) \leq 2\sigma \text{ or }\mu_0(\{x_t\})>\epsilon,\\
            x_\emptyset &\text{otherwise}.
        \end{cases}
    \]
    Let $\tilde\mu_t$ be the distribution of $\tilde x_t$ conditional on $x_{<t}$. As in the proof of \cref{lemma:coupling}, one checks that the density ratio is bounded by $4/\epsilon$ with respect to $\tilde\mu_0$, hence $\tilde x_1,\ldots,\tilde x_T$ is $(\epsilon/4)$-smooth.

    Next, for any $t\in[T]$, define
    \[A_t:=\Bigl\{x\in\Xcal: \tfrac{d\mu_t}{d\mu_0}(x) > 2\sigma\Bigr\}\cap\Bigl\{x\in\Xcal: \mu_0(\{x\})\leq \epsilon\Bigr\}.
    \]
    By construction,
    \[\Pbb[\tilde x_t=x_\emptyset\mid \tilde x_{<t},x_{<t}] = \mu_t(A_t).\]

    If $\mu_0(A_t)\leq \epsilon$, then by assumption $\mu_t(A_t)\leq \mu_\Ucal(A_t)\leq \eta$.
    Otherwise, suppose $\mu_0(A_t)>\epsilon$. By \cref{lemma:small_enough_set}, there exists a measurable $B\subseteq A_t$ with $\mu_0(B)\in(\epsilon/2,\epsilon]$. On the one hand, by the defining property of $A_t$ we have $\mu_t(B)\geq 2\sigma\,\mu_0(B) > \sigma\epsilon = \eta$. On the other hand, $\mu_0(B)\leq \epsilon$ implies $\mu_t(B)\leq \mu_\Ucal(B)\leq \eta$, a contradiction.
    Therefore $\mu_t(A_t)\leq \eta$ in all cases, completing the proof.
\end{proof}

\begin{lemma}\label{lemma:small_enough_set}
    Let $\epsilon>0$, let $\mu$ be a probability measure on $\Xcal$, and let $A\subseteq\Xcal$ be measurable such that $\mu(A)>\epsilon$ and for all $x\in A$ we have $\mu(\{x\})\leq \epsilon$. Then there exists a measurable set $B\subseteq A$ such that $\mu(B)\in(\epsilon/2,\epsilon]$.
\end{lemma}

\begin{proof}
    If there exists an atom $x\in A$ with $\mu(\{x\})\in(\epsilon/2,\epsilon]$, then take $B=\{x\}$.

    Otherwise, every atom in $A$ has mass at most $\epsilon/2$. Let $A_{\mathrm{at}}\subseteq A$ be the (at most countable) set of atoms and write $m:=\mu(A_{\mathrm{at}})\in[0,1]$.
    If $m\geq \epsilon/2$, pick a finite subset of atoms $\{x_1,\ldots,x_k\}\subseteq A_{\mathrm{at}}$ by adding atoms greedily until the partial sum first exceeds $\epsilon/2$. Since each atom has mass at most $\epsilon/2$, the resulting set $B=\{x_1,\ldots,x_k\}$ satisfies $\mu(B)\in(\epsilon/2,\epsilon]$.

    If $m<\epsilon/2$, then the non-atomic part $A\setminus A_{\mathrm{at}}$ has measure $\mu(A\setminus A_{\mathrm{at}})>\epsilon-m>\epsilon/2$. Since $\mu$ restricted to $A\setminus A_{\mathrm{at}}$ is non-atomic, there exists a measurable $C\subseteq A\setminus A_{\mathrm{at}}$ with $\mu(C)=\epsilon/2$. Taking $B:=A_{\mathrm{at}}\cup C$ yields $\mu(B)=m+\epsilon/2\in(\epsilon/2,\epsilon]$.
\end{proof}

\section{Proofs from \cref{sec:differential_privacy}}

\subsection{Proof of \cref{thm:private_main_result}}

We start by showing the following lemma which is a consequence of Tur\'an's theorem which lower bounds the independence number of graphs with few edges.

\begin{lemma}\label{lemma:turan_reduction}
    Fix $\epsilon,\delta\in(0,1]$, $N\geq 1$, and $\Ucal$ be a distribution class such that $N_\Ucal(\epsilon)\geq N$. Then, there exist $k\geq N\delta/2$ disjoint measurable sets $A_1,\ldots,A_k$ together with distributions $\mu_1,\ldots,\mu_k\in\Ucal$ such that
    \begin{enumerate}
        \item for all $i\in[k]$, $\mu_i(A_i)\geq \epsilon$,
        \item and for all $i\neq j\in[k]$, $\mu_i(A_j) \leq \delta$.
    \end{enumerate}
    Additionally, if $N_\Ucal(\epsilon)=\infty$, then there exist disjoint measurable sets $(A_i)_{i\geq 1}$, distributions $(\mu_i)_{i\geq 1}$ in $\Ucal$ such that for all $i\geq 1$, $\mu_i(A_i)\geq \epsilon$, and for all $n\geq 1$,
    \begin{equation*}
        \sup_{2^n\leq i\neq j< 2^{n+1}}\mu_i(A_j) \leq 2^{-n}.
    \end{equation*}
\end{lemma}

\begin{proof}
    Fix $\Ucal$ such that $N_\Ucal(\epsilon)\geq N$. Then there exist $B_1,\ldots,B_N$ disjoint measurable sets together with distributions $\nu_1,\ldots,\nu_N\in\Ucal$ with $\nu_i(B_i)\geq \epsilon$ for all $i\in[N]$. Construct the graph $\Gcal=([N],E)$ on $[N]$ such that $(i,j)\in E$ if and only if $\nu_i(B_j)\geq \delta$ or $\nu_j(B_i)\geq \delta$ for some $\delta\in(0,1]$. Note that $\Gcal$ has at most $N/\delta$ edges since for each $i\in[N]$ there can be at most $1/\delta$ indices $j\in[N]$ for which $\nu_i(B_j)\geq \delta$ given that all sets $B_j$ are disjoint. Therefore, $\Gcal$ has average node degree at most $1/\delta$. Hence, Tur\'an's theorem shows that there is an independent set $S$ of size at least $N/(1+1/\delta) \geq N\delta/2$. Focusing only on the indices $i\in S$ with the corresponding sets $B_i$ and distribution $\nu_i$ gives the first desired result.

    We now prove the second claim when $N_\Ucal(\epsilon)=\infty$. Fix a sequence $(B_j)_{j\geq 1}$ of disjoint measurable sets and corresponding distributions $(\nu_j)_{j\geq 1}$ in $\Ucal$ such that $\nu_j(B_j)\geq \epsilon$ for all $i\geq 1$. We then essentially stitch together the above construction for the parameters $\delta_n=2^{-n}$ for $n\geq 1$. Specifically, having constructed indices $j_i\in\Nbb$ for all $i<2^n$, we focus on indices that have not been used yet $\Ical_n:=\Nbb\setminus \{j_l:l<2^n\}$. Among the first $2^{2n}$ elements of $\Ical_n$, the previous arguments show that we can construct distinct indices $i_j\in\Ical_l$ for $j\in [2^n,2^{n+1})$ such that letting $A_j=B_{i_j}$ and $\mu_j=\nu_{i_j}$ we have $\nu_i(A_j)\leq 2^{-n}$ for all $i\neq j\in[2^n,2^{n+1})$. This proves the desired properties and ends the proof.
\end{proof}

With this tool at hand we now show that algorithms that privately learn generalized thresholds under $\Ucal$ distributions at scale $\epsilon$ with sample complexity $m$, induce private algorithms for learning thresholds of size $\Omega(\sqrt{N_\Ucal(\epsilon)/m})$.

\begin{lemma}\label{lemma:reduction_to_empirical_learner}
    Let $\Ucal$ be a convex distribution class. Fix $m\geq 1$ and $\eta\in(0,1]$. If for any generalized threshold function class $\Fcal$, there exists an $(\epsilon,\delta)$-differentially private $(\alpha,\beta)$-accurate algorithm on distributions in $\Ucal$ with sample complexity $m$, then there exists an $(\epsilon,\delta)$-differentially private $(2\alpha/\eta,2\beta)$-accurate PAC learner for thresholds in $[K]$ with sample complexity $m$, where $K=\Omega\paren{\sqrt{\frac{\beta N_\Ucal(\eta)}{m}}}$.

    Next, suppose that $N_\Ucal(\eta)=\infty$ for some $\eta\in(0,1]$. Then, if for any generalized threshold function class $\Fcal$, there exists an $(\epsilon,\delta)$-differentially private $(\alpha,\beta)$-accurate algorithm on distributions in $\Ucal$ with some sample complexity $m(\Fcal)<\infty$, then there exists $m<\infty$ and an $(\epsilon,\delta)$-differentially private $(2\alpha/\eta,2\beta)$-accurate PAC learner for thresholds in $[K]$ with sample complexity $m$ for any $K\geq 1$.
\end{lemma}

\begin{proof}
    Let $\gamma\in(0,1]$ be a parameter to be fixed later. We first use \cref{lemma:turan_reduction} to construct $K=\floor{N_\Ucal(\eta)\gamma/2}$ disjoint measurable sets $A_i$ and distributions $\mu_i\in\Ucal$ for $i\in[K]$ satisfying $\mu_i(A_i)\geq \eta$ and $\mu_i(A_j)\leq \gamma$ for all $i\neq j\in[K]$. We now consider the generalized threshold class on these sets $\Fcal:=\{x\in\Xcal\mapsto \1[x\in\bigcup_{j\geq i} A_j]: i\in[K+1]\}$. 

    Given an $(\epsilon,\delta)$-differentially private $(\alpha,\beta)$-accurate algorithm for $\Fcal$ on distributions in $\Ucal$ with sample complexity $m$, we construct a new algorithm $\Acal'$ as follows. For any dataset $\Dcal' = (a_i,y_i)_{i\in[m]}\in ([K]\times\{0,1\})^m$, we replace each datapoint $(a_i,y_i)$ with $(x_i,y_i':=y_i\1[x_i\in A_i])$ where $x_i\sim \mu_{a_i}$ are sampled independently for each $i\in[m]$. Denote by $\Dcal$ the corresponding dataset and let $h$ be the output of $\Acal$ on $\Dcal$. $\Acal'$ returns the vector $(\1[\Ebb_{x\sim\mu_i}[h(x)\mid x\in A_i]\geq 1/2])_{i\in[K]}$. Since $\Acal$ is $(\epsilon,\delta)$-private, so is the algorithm that would directly output the hypothesis $h$ (note that $\Acal'$ can be viewed as a mixture of $(\epsilon,\delta)$-private algorithms, one for each potential realization of $x_i\sim\mu_i$ for $i\in[K]$). Hence $\Acal'$ is also $(\epsilon,\delta)$-private.

    We next show that $\Acal'$ is a PAC learner for the class of thresholds on $[K]$ which we denote by $\Fcal':=\{i\in[K]\mapsto \1[i\geq i_0]: i_0\in[K+1]\}$. Fix any distribution $\nu$ on $[K]$ and a threshold $g'\in\Fcal$. Let $\Dcal'=(a_i,y_i)_{i\in[m]}$ be a dataset corresponding of $m$ i.i.d.\ samples from $\nu$ labeled by $g'$. Note that for any sample $i\in[m]$, letting $x_i\sim\mu_{a_i}$ be the sampled instance, if $x_i\notin \bigcup_{j\neq i}A_j$, then the value $y_i\1[x_i\in A_i]$ is consistent with the function $g:x\in\Xcal\mapsto \1[x\in\bigcup_{i\in[K]:g'(i)=1} A_i]\in\Fcal$. In particular, under the event
    \begin{equation*}
        \Ecal:=\left\{\forall i\in[m]: x_i\notin \bigcup_{j\neq i}A_j \right\},
    \end{equation*}
    $\Dcal'$ coincides with a dataset composed of $m$ i.i.d.\ samples from the mixture $\mu:=\sum_{i\in[K]} \nu_i \mu_i$ labeled by $g$. Note that $\mu\in\Ucal$ since $\Ucal$ is convex and that by construction,
    \begin{equation*}
        \Pbb[\Ecal] \leq mK\gamma.
    \end{equation*}
    Since $\Acal$ is $(\alpha,\beta)$-accurate, this shows that with probability at least $1-\beta-mK\gamma$ on the sampled dataset $\Dcal'$, the output $h$ of $\Acal$ on the corresponding dataset $\Dcal$ satisfies
    \begin{align*}
        \alpha &\geq \Pbb_{x\sim\mu}[h(x)\neq g(x)]\\
        &\geq \Ebb_{i\sim\nu}\sqb{\mu_i(A_i) \Pbb_{x\sim \mu_i}[h(x)\neq g'(i)\mid x\in A_i]}\\
        &\geq \frac{\eta}{2}\Ebb_{i\sim\nu}[h'(i)\neq g'(i)],
    \end{align*}
    where $h'$ is the final output of $\Acal'$. This shows that $\Acal'$ is an $(2\alpha/\eta,\beta+mK\gamma)$-accurate PAC learner for $\Fcal'$. Note that $mK\gamma\leq m N_\Ucal(\eta)\gamma^2/2$. Taking $\gamma=\sqrt{\frac{2\beta}{mN_\Ucal(\eta)}}$ gives the desired result for the first claim.
    
    The second claim is proved analogously. Suppose that $N_\Ucal(\eta)=\infty$. Then, we use \cref{lemma:turan_reduction} to construct disjoint measurable sets $A_i$ and distributions $\mu_i\in\Ucal$ for $i\geq 1$ satisfying $\mu_i(A_i)\geq \eta$ for all $i\geq 1$, and for all $n\geq 1$ and $i\neq j\in[2^n,2^{n+1})$, $\mu_i(A_j)\leq 2^{-n}$. As above, we construct the generalized threshold class on these sets $\Fcal:=\{x\in\Xcal\mapsto \1[x\in\bigcup_{j\geq i} A_j]: i\geq 1\}$. The main point is that for any $n\geq 1$, this function class essentially encapsulates the previous construction for$\gamma_n=2^{-n}$, by focusing only on the subclass $\Fcal_n:=\{x\in\Xcal\mapsto \1[x\in\bigcup_{j\geq i} A_j]: i\in[2^n,2^n+K_n]\}$ where $K_n:=\floor{2^n/m}$. The only slight modification in the construction of the PAC learner $\Acal'$ for the thresholds on $[K_n]$ is that we additionally always label $y_i'=1$ if $x_i\in \bigcup_{j<2^n}A_j$, to preserve realizability in that case. Following the exact same arguments therefore shows that if there is an $(\epsilon,\delta)$-differentially private $(\alpha,\beta)$-accurate algorithm for $\Fcal$ on distributions in $\Ucal$ with sample complexity $m$, then for any $n\geq 1$, there is an $(2\alpha/\eta,\beta+mK_n\gamma_n\leq 2\beta)$-accurate PAC learner for $\Fcal'$. This ends the proof.
\end{proof}

Combining this result with the sample complexity lower bound from \cite[Theorem 1]{alon2019private} for privately learning thresholds yields the following result.

\begin{theorem}\label{thm:private_result}
    There is a constant $c_0>0$ such that the following holds. Let $\Ucal$ be a convex distribution class, $\epsilon\in(0,1]$, and $m\geq 1$. If for any generalized threshold class $\Fcal$, there exists a $(0.1,\frac{c_0}{m^2\log m})$-differentially private $(\epsilon/32,1/32)$-accurate algorithm on distributions from $\Ucal$, with sample complexity $m$; then
    \begin{equation*}
        m \geq \Omega(\log^\star N_\Ucal(\epsilon)).
    \end{equation*}
    
    Also, if $\Ucal$ is a convex distribution such that $N_\Ucal(\epsilon)=\infty$ for some $\epsilon\in(0,1]$, then there exists a generalized threshold class $\Fcal$ such that for any $m\geq 1$, there does not exist a $(0.1,\frac{c_0}{m^2\log m})$-differentially private $(\epsilon/32,1/32)$-accurate algorithm on distributions from $\Ucal$, with sample complexity $m$.
\end{theorem}

\begin{proof}
    Given such an algorithm for generalized threshold classes, \cref{lemma:reduction_to_empirical_learner} constructs a $(0.1,\frac{c_0}{m^2\log m}))$-differentially private $(1/16,1/16)$-accurate PAC learner for thresholds on $[K]$ where $K=\Omega(\sqrt{N_\Ucal(\epsilon)/m})$. Therefore, for $c_0>0$ sufficiently small, \cite[Theorem 1]{alon2019private} shows that
    \begin{equation*}
        m\geq \Omega(\log^\star K) = \Omega(\log^\star N_\Ucal(\epsilon) - \log^\star m).
    \end{equation*}
    In turn, this shows the desired bound $m\geq \Omega(\log^\star N_\Ucal(\epsilon))$.

    The second claim is proved similarly using the second claim from \cref{lemma:reduction_to_empirical_learner}.
\end{proof}

We are now ready to prove \cref{thm:private_main_result} as an immediate consequence of \cref{thm:private_result} and the previous characterization of generalized-smooth distribution classes.

\vspace{3pt}

\begin{proof}{\textbf{of \cref{thm:private_main_result}}}
    If $\Ucal$ is not generalized smooth, then \cref{lemma:continuous_submeasure_to_smooth} implies that $\mu_\Ucal$ is not a continuous submeasure, which in turn implies that there is $\epsilon>0$ for which $N_\Ucal(\epsilon)=\infty$. Then, the desired impossibility result for private learnability is a direct consequence from \cref{thm:private_result}.
\end{proof}


\subsection{Proof of \cref{thm:private_sufficiency}}

The proof follows the same structure as private learning algorithms for VC classes under smoothness and with unlabelled public data \citep{haghtalab2020smoothed,alon2019limits} but we include it here for completeness.
First, recall that under generalized smoothness, that any VC class has finite uniform cover from  \cref{def:unif_covers}. 
We present this result below for completeness.

\begin{lemma}[Uniform Covers under Generalized Smoothness] \label{lem:unif_cover}
    Let $\Ucal$ be a $ (\rho, \mu_0) $ generalized smooth class of distributions on $\Xcal$. Then, for any function class $\Fcal$ of VC dimension $d$, any $\epsilon\in(0,1]$, there is a class $\Fcal_\epsilon$ which is an $\epsilon$-uniform cover of $\Fcal$ with respect to $\Ucal$ and satisfies
    \begin{equation*}
        \log  |\Fcal_\epsilon| \leq d \log( C \cdot \rho^{-1}(  \epsilon) ) 
    \end{equation*}
    for some universal constant $C>0$. 
\end{lemma}

\begin{proof}
    The idea is to construct an $\epsilon'$-cover of $\Fcal$ with respect to the base measure $\mu_0$. 
    For any $f,f'$ and $\mu \in \Ucal$ we have that $\mu \left( \left\{ x : f(x) \neq f'(x) \right\} \right) \leq \rho( \mu_0 \left( \left\{ x : f(x) \neq f'(x) \right\} \right)  )  \leq \rho(\epsilon' ) $. Setting $\epsilon' = \rho^{-1}(\epsilon)$, we see that any $\epsilon'$-cover with respect to $\mu_0$ is an $\epsilon$-uniform cover with respect to $\Ucal$. The result then follows from the classical VC dimension covering number bound.    
\end{proof}

The required private learning algorithm is then a direct application of the private learning algorithm based on the exponential mechanism from \cite{kasiviswanathan2011efficient} using uniform covers. We recall the main result below for completeness.

\begin{theorem}[\cite{kasiviswanathan2011efficient}]\label{thm:exp_mech_private_learning}
    Let $\Fcal$ be a finite function class. For any $\epsilon,\delta\in(0,1]$, there exists an $(\alpha,\beta)$-differentially private algorithm that is $(\alpha,\beta)$-accurate for learning $\Fcal$ with sample complexity
    \begin{equation*}
        m = \frac{d + \log(1/\delta) }{ \epsilon^2 } + \frac{ \log |\Fcal| }{ \alpha \epsilon }.
    \end{equation*}
\end{theorem}

and note that a $\left( \alpha, 0 \right)$  private algorithm $\left( \epsilon , \delta \right)$ accurate learning algorithm for $\Fcal_{\epsilon}$ is also a $\left( \alpha, 0 \right)$  private algorithm $\left( 2\epsilon , \delta \right)$ accurate learning algorithm for $\Fcal$ for any distribution with marginal in $\Ucal$.

\end{document}

%% file: shortcuts.tex
\newcommand{\trw}{\text{\small TRW}}
\newcommand{\maxcut}{\text{\small MAXCUT}}
\newcommand{\maxcsp}{\text{\small MAXCSP}}
\newcommand{\suol}{\text{SUOL}}
\newcommand{\wuol}{\text{WUOL}}
\newcommand{\crf}{\text{CRF}}
\newcommand{\sual}{\text{SUAL}}
\newcommand{\suil}{\text{SUIL}}
\newcommand{\fs}{\text{FS}}
\newcommand{\fmv}{{\text{FMV}}}
\newcommand{\smv}{{\text{SMV}}}
\newcommand{\wsmv}{{\text{WSMV}}}
\newcommand{\trwp}{\text{\small TRW}^\prime}
\newcommand{\rhos}{\rho^\star}
\newcommand{\brhos}{\brho^\star}
\newcommand{\bzero}{{\mathbf 0}}
\newcommand{\bs}{{\mathbf s}}
\newcommand{\bw}{{\mathbf w}}
\newcommand{\bws}{\bw^\star}
\newcommand{\ws}{w^\star}
\newcommand{\Prt}{{\mathsf {Part}}}
\newcommand{\Fs}{F^\star}

\newcommand{\Hs}{{\mathsf H} }

\newcommand{\hL}{\hat{L}}
\newcommand{\hU}{\hat{U}}
\newcommand{\hu}{\hat{u}}

\newcommand{\bu}{{\mathbf u}}
\newcommand{\ubf}{{\mathbf u}}
\newcommand{\hbu}{\hat{\bu}}

\newcommand{\primal}{\textbf{Primal}}
\newcommand{\dual}{\textbf{Dual}}

\newcommand{\Ptree}{{\sf P}^{\text{tree}}}
\newcommand{\bv}{{\mathbf v}}

\newcommand{\bq}{\boldsymbol q}

\newcommand{\rvM}{\text{M}}

\newcommand{\Acal}{\mathcal{A}}
\newcommand{\Bcal}{\mathcal{B}}
\newcommand{\Ccal}{\mathcal{C}}
\newcommand{\Dcal}{\mathcal{D}}
\newcommand{\Ecal}{\mathcal{E}}
\newcommand{\Fcal}{\mathcal{F}}
\newcommand{\Gcal}{\mathcal{G}}
\newcommand{\Hcal}{\mathcal{H}}
\newcommand{\Ical}{\mathcal{I}}
\newcommand{\Jcal}{\mathcal{J}}
\newcommand{\Kcal}{\mathcal{K}}
\newcommand{\Lcal}{\mathcal{L}}
\newcommand{\Mcal}{\mathcal{M}}
\newcommand{\Ncal}{\mathcal{N}}
\newcommand{\Pcal}{\mathcal{P}}
\newcommand{\Scal}{\mathcal{S}}
\newcommand{\Tcal}{\mathcal{T}}
\newcommand{\Ucal}{\mathcal{U}}
\newcommand{\Vcal}{\mathcal{V}}
\newcommand{\Wcal}{\mathcal{W}}
\newcommand{\Xcal}{\mathcal{X}}
\newcommand{\Ycal}{\mathcal{Y}}
\newcommand{\Ocal}{\mathcal{O}}
\newcommand{\Qcal}{\mathcal{Q}}
\newcommand{\Rcal}{\mathcal{R}}

\newcommand{\brho}{\boldsymbol{\rho}}

\newcommand{\Cbb}{\mathbb{C}}
\newcommand{\Ebb}{\mathbb{E}}
\newcommand{\Nbb}{\mathbb{N}}
\newcommand{\Pbb}{\mathbb{P}}
\newcommand{\Qbb}{\mathbb{Q}}
\newcommand{\Rbb}{\mathbb{R}}
\newcommand{\Sbb}{\mathbb{S}}
\newcommand{\Vbb}{\mathbb{V}}
\newcommand{\Wbb}{\mathbb{W}}
\newcommand{\Xbb}{\mathbb{X}}
\newcommand{\Ybb}{\mathbb{Y}}
\newcommand{\Zbb}{\mathbb{Z}}

\newcommand{\Rbbp}{\Rbb_+}

\newcommand{\bX}{{\mathbf X}}
\newcommand{\bx}{{\boldsymbol x}}

\newcommand{\btheta}{\boldsymbol{\theta}}

\newcommand{\Pb}{\mathbb{P}}

\newcommand{\hPhi}{\widehat{\Phi}}

\newcommand{\Sigmah}{\widehat{\Sigma}}
\newcommand{\thetah}{\widehat{\theta}}

\newcommand{\indep}{\perp \!\!\! \perp}
\newcommand{\notindep}{\not\!\perp\!\!\!\perp}

\newcommand{\one}{\mathbbm{1}}
\newcommand{\1}{\mathbbm{1}}
\newcommand{\aprx}{\alpha}

\newcommand{\ST}{\Tcal(\Gcal)}
\newcommand{\x}{\mathsf{x}}
\newcommand{\y}{\mathsf{y}}
\newcommand{\Ybf}{\textbf{Y}}
\newcommand{\smiddle}[1]{\;\middle#1\;}

\definecolor{dark_red}{rgb}{0.2,0,0}
\newcommand{\detail}[1]{\textcolor{dark_red}{#1}}

\newcommand{\ds}[1]{{\color{red} #1}}
\newcommand{\rc}[1]{{\color{green} #1}}

\newcommand{\mb}[1]{\ensuremath{\boldsymbol{#1}}}

\newcommand{\metric}{\rho}
\newcommand{\proj}{\text{Proj}}

\newcommand{\paren}[1]{\left( #1 \right)}
\newcommand{\sqb}[1]{\left[ #1 \right]}
\providecommand{\set}[1]{\left\{ #1 \right\}}
\newcommand{\floor}[1]{\left\lfloor #1 \right\rfloor}
\newcommand{\ceil}[1]{\left\lceil #1 \right\rceil}
\newcommand{\abs}[1]{\left|#1\right|}
\newcommand{\norm}[1]{\left\|#1\right\|}

\newcommand{\Ber}{\textnormal{Ber}}


\newcommand{\domain}{\mathcal{X}}
\newcommand{\distfam}{\mathcal{U}} 